\pdfoutput=1
\documentclass[11pt]{article} 
\def\arXiv{1} 

\usepackage{enumitem}
\usepackage{amssymb}
\usepackage{amsbsy}
\usepackage{amsmath}
\usepackage{amsfonts}
\usepackage{latexsym}
\usepackage{graphicx}
\usepackage{color}
\usepackage{xifthen}
\usepackage{xspace}
\usepackage{mathtools}
\usepackage{bbm}

\usepackage{multirow}
\usepackage[normalem]{ulem}

\usepackage{array}
\usepackage[T1]{fontenc}

\usepackage{etoolbox}
\newtoggle{restatements}

\newtoggle{heavyplots}

\usepackage{xfrac}
\usepackage[algo2e,ruled]{algorithm2e}

\newcommand{\notarxiv}[1]{foo}
\newcommand{\arxiv}[1]{ba}
\ifdefined \arXiv
\renewcommand{\arxiv}[1]{#1}%
\renewcommand{\notarxiv}[1]{\ignorespaces}%
\else%
\renewcommand{\arxiv}[1]{\ignorespaces}%
\renewcommand{\notarxiv}[1]{#1}%
\fi%

\notarxiv{

	\usepackage[subtle, all=normal, mathspacing=tight, floats=tight, 
	sections=tight,wordspacing=normal,mathdisplays=normal]{savetrees}
}

\arxiv{
	\usepackage[top=1in, right=1in, left=1in, bottom=1in]{geometry}
	
	\usepackage[dvipsnames]{xcolor}
	
	\usepackage{amsthm}
	\usepackage{natbib}
	\bibliographystyle{plainnat}
	\bibpunct{(}{)}{;}{a}{,}{,}
	\usepackage{url}
	\usepackage[hidelinks]{hyperref}
	\hypersetup{
		colorlinks=true,
		linkcolor=blue!70!black,
		citecolor=blue!70!black,
		urlcolor=blue!70!black}
}

\usepackage{thmtools}
\usepackage{thm-restate}
\usepackage{cleveref}

\arxiv{
\theoremstyle{plain}

\newtheorem{theorem}{Theorem}
\newtheorem{lemma}{Lemma}

\newtheorem{proposition}{Proposition}
\newtheorem{corollary}{Corollary}

}

\newtheorem{assumption}{Assumption}

\notarxiv{
	\declaretheorem[name=Lemma,sibling=theorem]{lem} %
}
\arxiv{
}

\makeatletter
\newcommand*{\propenum}[1]{%
	\expandafter\@propenum\csname c@#1\endcsname%
}

\newcommand*{\@propenum}[1]{%
	$\ifcase#1\or1\or1'\or2\or3\or42%
	\else\@ctrerr\fi$%
}
\AddEnumerateCounter{\propenum}{\@propenum}{1}
\makeatother

\newcommand{\wt}[1]{\widetilde{#1}}  %
\newcommand{\wb}[1]{\overline{#1}} %

\newcommand{\norm}[1]{\left\|{#1}\right\|} %
\newcommand{\lone}[1]{\norm{#1}_1} %
\newcommand{\ltwo}[1]{\norm{#1}_2} %
\newcommand{\linf}[1]{\norm{#1}_\infty} %
\newcommand{\norms}[1]{\|{#1}\|} %

\newcommand{\opnorm}[1]{\linf{#1}}  %
\newcommand{\lones}[1]{\norms{#1}_1} %
\newcommand{\ltwos}[1]{\norms{#1}_2} %
\newcommand{\linfs}[1]{\norms{#1}_\infty} %
\newcommand{\R}{\mathbb{R}} %
\newcommand{\N}{\mathbb{N}} %
\newcommand{\E}{\mathbb{E}} %
\renewcommand{\P}{\mathbb{P}}	%
\newcommand{\I}{\mathbb{I}} %

\providecommand{\argmax}{\mathop{\rm argmax}} %

\providecommand{\dom}{\mathop{\rm dom}}
\providecommand{\diag}{\mathop{\rm diag}}
\providecommand{\tr}{\mathop{\rm tr}}

\providecommand{\interior}{\mathop{\rm relint}}

\providecommand{\minimize}{\mathop{\rm minimize}}

\newcommand{\ceil}[1]{\left\lceil{#1}\right\rceil}

\newcommand{\half}{\frac{1}{2}}

\newcommand{\defeq}{\coloneqq}

\newcommand{\grad}{\nabla}
\newcommand{\hess}{\nabla^2}

\newcommand{\eps}{\epsilon}

\newcommand{\Otil}[1]{\widetilde{O}( #1 )}

\newcommand{\expect}{\operatorname{\E}\expectarg}
\DeclarePairedDelimiterX{\expectarg}[1]{[}{]}{%
	\ifnum\currentgrouptype=16 \else\begingroup\fi
	\activatebar#1
	\ifnum\currentgrouptype=16 \else\endgroup\fi
}

\newcommand{\innermid}{\nonscript\;\delimsize\vert\nonscript\;}
\newcommand{\activatebar}{%
	\begingroup\lccode`\~=`\|
	\lowercase{\endgroup\let~}\innermid 
	\mathcode`|=\string"8000
}

\newcommand{\domain}{\varDelta_{n}} %
\newcommand{\domainDual}[1][n]{S_{#1}}
\newcommand{\simplex}[1][m]{\sigma_{#1}}

\newcommand{\<}{\langle}
\renewcommand{\>}{\rangle}
\newcommand{\inner}[2]{\left<#1,#2\right>}
\newcommand{\inners}[2]{\big<#1,#2\big>}
\newcommand{\innerB}[2]{\Bigg<#1,#2\bigg>}
\newcommand{\breg}[2]{\bar{V}_{#1}(#2)}

\newcommand{\bregBlank}{\bar{V}}

\newcommand{\x}[1][t]{
  \ifthenelse{\isempty{#1}}{%
    X
  }{%
    X_{#1}
  }
}

\newcommand{\xAv}[1][t]{\bar{X}_{#1}}
\newcommand{\xApp}[1][t]{\tilde{X}_{#1}}
\newcommand{\xM}{Y}
\newcommand{\dM}{D}  %
\newcommand{\g}[1][t]{G_{#1}}

\newcommand{\indic}[1]{\I_{\{#1\}}}

\newcommand{\regret}[1][T]{\mathcal{R}[#1]}

\newcommand{\srank}[1][]{
  \ifthenelse{\isempty{#1}}{%
    \mathsf{sr}
  }{%
    \mathsf{sr}\!\left({#1}\right)
  }
}

\newcommand{\Filt}[1][t-1]{\mathcal{F}_{#1}}

\newcommand{\Map}[1][u]{\mathsf{P}_{#1}}
\newcommand{\MapAv}{\bar{\mathsf{P}}}
\newcommand{\mapAv}{\bar{\mathsf{p}}}

\newcommand{\mapConj}{\bar{\mathsf{r}}}
\newcommand{\MapQ}{\mathsf{P}^{\tiny{\textup{q-reg}}}}
\newcommand{\MapMMW}{\mathsf{P}^{\tiny{\mathrm{mw}}}}
\newcommand{\mapMMW}{\mathsf{p}^{\tiny\mathrm{mw}}}

\newcommand{\ma}{A}
\newcommand{\lanczos}{\wt{\exp}_k}
\newcommand{\MapApp}[1][u;k]{\wt{\mathsf{P}}_{#1}}

\newcommand{\softmax}{{\rm{lse}}}
\newcommand{\gradsoftmax}{\grad \softmax}

\newcommand{\smapMMW}{\softmax}
\newcommand{\smapAv}{\wb{\softmax}}
\newcommand{\AMMW}{A^{\ones}}
\newcommand{\AAv}{\bar{A}}

\newcommand{\uniform}{\mathsf{Uni}}
\newcommand{\betadist}{\mathrm{Beta}}
\newcommand{\sphere}{\mathbb{S}}
\newcommand{\ones}{\mathbf{1}}

\newcommand{\dirichlet}{\mathrm{Dirichlet}(\half,\ldots,\half)}

\newcommand{\matvec}{\mathrm{mv}}
\newcommand{\lambdamin}{\lambda_{\min}}
\newcommand{\lambdamax}{\lambda_{\max}}

\newcommand{\Sset}{\mathcal{S}}

\newcommand{\saddlevalue}{\mathfrak{s}}
\newcommand{\Xavg}{X^{\textup{avg}}_T}
\newcommand{\yavg}{y^{\textup{avg}}_T}
\newcommand{\nnz}{\mathrm{nnz}}
\newcommand{\timeA}{\nnz(\mathcal{A})}

\newcommand{\timeMatvecA}{\matvec(\mathcal{A})}
\newcommand{\timeMatvecAstar}{\matvec(\mathcal{A}^\star)}
\newcommand{\dualitygap}{\mathrm{Gap}}

\newcommand{\yair}[1]{}
\newcommand{\jcd}[1]{}
\newcommand{\sidford}[1]{}
\newcommand{\kevin}[1]{}

\title{A Rank-1 Sketch 
for Matrix Multiplicative Weights}

\author{Yair Carmon ~~~ John C.\ Duchi ~~~  Aaron Sidford ~~~ Kevin 
Tian\\
	\texttt{\{\href{mailto:yairc@stanford.edu}{yairc},%
		\href{mailto:jduchi@stanford.edu}{jduchi},%
		\href{mailto:sidford@stanford.edu}{sidford},%
		\href{mailto:kjtian@stanford.edu}{kjtian}\}@stanford.edu}}
\date{}

\begin{document}

\maketitle

\begin{abstract}%
 We show that a simple randomized sketch of the matrix multiplicative 
 weight (MMW) update enjoys (in expectation) the same regret bounds as 
 MMW, up to a small constant factor. Unlike MMW, where every step 
 requires full matrix exponentiation, our steps require only a single product 
 of the form $e^A b$, which the Lanczos method approximates efficiently. 
 Our key technique is to view the sketch as a \emph{randomized mirror 
 projection}, and perform mirror descent analysis on the \emph{expected 
 projection}. Our sketch solves the online eigenvector problem, improving 
 the best known complexity bounds by $\Omega(\log^5 n)$. We also apply 
 this sketch to semidefinite programming in saddle-point form, yielding a 
 simple primal-dual scheme with guarantees matching the best in the 
 literature.
\end{abstract}

\section{Introduction}\label{sec:intro}

Consider the problem of online learning over the 
spectrahedron $\domain$, the set of $n\times n$ symmetric positive 
semidefinite  
matrices with unit trace.
  At every time 
step $t$, a player chooses action $\x\in\domain$, an adversary supplies 
symmetric gain matrix $\g$, and the player earns reward 
$\inner{\g}{\x}\defeq \tr( \g \x)$. 
We seek to minimize the regret with respect to the best single action (in 
hindsight),
\begin{equation}\label{eq:regret-def}
\sup_{\x[] \in \domain}\sum_{t=1}^T\inners{\g}{\x[]}
-
\sum_{t=1}^T\inners{\g}{\x} 
= 
\lambdamax\Bigg(\sum_{t=1}^T\g\Bigg) - \sum_{t=1}^T\inners{\g}{\x}.
\end{equation}
\citet{WarmuthKu08,WarmuthKu12} solve this problem using the matrix 
exponentiated gradient algorithm~\citep{TsudaRaWa05}, also known as 
matrix multiplicative 
weights (MMW). It is given by
\begin{equation}\label{eq:mmw-def}
X_t = \MapMMW\left( \eta \sum_{i=1}^{t-1}\g[i]\right),
~~\mbox{where}~~
\MapMMW(\xM) \defeq \frac{e^{\xM}}{\tr e^{\xM}},
\end{equation}
and $\eta > 0$ is a step size parameter. If the operator norm $\linf{\g}\le 
1$ for every $t$, 
using the MMW strategy~\eqref{eq:mmw-def} with $\eta = 
\sqrt{2\log (n) /T}$ guarantees that the regret~\eqref{eq:regret-def} is 
bounded by $\sqrt{2\log(n) T}$; this guarantee is minimax optimal up to 
a constant~\citep{AroraHaKa12}. 

Unlike standard (vector) multiplicative weights, MMW is 
computational expensive to implement in the high-dimensional setting $n\gg 1$. This is due ot the high cost of computing matrix 
exponentials; currently they require an 
eigen-decomposition which costs $\Theta(n^3)$ with practical 
general-purpose 
methods and $\Omega(n^\omega)$ in theory~\citep{PanChZh99}. This 
difficulty has led a number of researchers to consider a rank-$k$ sketch 
of $\MapMMW$ of the form
\begin{equation}\label{eq:rank-k-sketch-def}
\Map[U](\xM) \defeq \frac{ e^{\xM/2} U U^T 
e^{\xM/2}}{\inner{e^{\xM}}{UU^T}},
~~\mbox{where}~~
U\in\R^{n\times k}
\end{equation}
and the elements of $U$ are i.i.d.\ standard Gaussian. For $k\ll n$, 
$\Map[U]$ 
is much cheaper than 
$\MapMMW$ to compute, since its computation requires only $k$ products 
of the form $e^{A}b$ which can be evaluated efficiently via iterative 
methods (see Section~\ref{sec:expvec-compute}). 
Since we play rank-deficient matrices, an adversary with knowledge of $\x$ 
may choose the gain $\g$ to be in its nullspace, incurring regret linear in 
$T$. 
To rule such an adversary out, we assume that $\g$ and $\x$ must be 
chosen 
simultaneously. We formalize this as
\begin{assumption}\label{ass:bandit-adversary}
	Conditionally on $\x[1],\g[1],\ldots,\x[t-1],\g[t-1]$, the gain $\g$ is 
	independent of $\x$.
\end{assumption}
\noindent
This assumption is standard in the literature on adversarial bandit 
problems~\citep{BubeckCe12} where it is similarly unavoidable. While it 
comes at 
significant loss of generality, Assumption~\ref{ass:bandit-adversary} holds 
in two  
important applications, as described below.

\paragraph{The challenge of bias}
Assumption~\ref{ass:bandit-adversary} allows us to write 
\[
\expect*{ 
\innerB{\g}{\Map[U_t]\bigg(\eta\sum_{i=1}^{t-1}\g[i]\bigg)} |  
\{\g[i]\}_{i=1}^t}= 
\innerB{G_t}{\E_{U} \Map[U]\bigg(\eta\sum_{i=1}^{t-1}\g[i]\bigg)}.\]
 However, even  
though $U$ satisfies $\E_U UU^T = I$, we have
 $\E_U \Map[U](\xM)  \ne \MapMMW(\xM)$ for general $\xM$.
Therefore, the guarantees of MMW do not immediately apply to actions 
chosen according to the sketch~\eqref{eq:rank-k-sketch-def}, 
 even 
in expectation. 
A common solution in the 
literature~\citep{AroraKa07,PengTaZh16,AllenLeOr16} is to pick 
$k=\Otil{1/\eps^2}$ 
such that, by the Johnson-Lindenstrauss lemma, $\Map[U](\xM)$ 
approximates $\MapMMW(\xM)$ to within multiplicative 
error $\eps$. This makes the MMW guarantees applicable again, but 
requires considerable computation per step, that will match the cost of full 
matrix exponentiation for sufficiently small $\eps$. 
\citet{KalaiVe05} and 
\citet{AllenLi17} prove 
regret guarantees for sketches of fixed rank $k\le3$ with forms different 
from~\eqref{eq:rank-k-sketch-def};
we discuss their approaches in detail in Section~\ref{sec:related}.

\paragraph{Our approach} In this work we use the 
sketch~\eqref{eq:rank-k-sketch-def} with $k=1$, 
playing the rank-1 matrix $\x = \Map[u_t](\eta\sum_{i=1}^{t-1}\g[i])$ 
where 
$\Map(\xM) = vv^T/(v^T v)$ for $v=e^{\xM/2}u$ and $u_t\in\R^n$ 
standard 
Gaussian. Instead of viewing $\Map$ as a biased 
estimator of $\MapMMW$, we define the deterministic function
\begin{equation*}
\MapAv(\xM) \defeq \E_u \Map(\xM),
\end{equation*}
and view $\Map$ as an unbiased estimator for $\MapAv$. 
Our primary contribution is in showing that 
\begin{equation*}
\text{$\MapAv$ is nearly as good a mirror projection as $\MapMMW$}.
\end{equation*}
More precisely, we show that replacing $\MapMMW$ with $\MapAv$ leaves the 
regret bounds 
almost unchanged; 
if $\linf{\g}\le1$ for every $t$, the 
 actions 
$\xAv = \MapAv(\eta\sum_{i=1}^{t-1}\g[i])$ guarantee (with properly 
tuned $\eta$) regret of at most $\sqrt{6\log(4n)T}$, worse than MMW 
by 
only a factor of roughly $\sqrt{3}$.
 To prove this, we establish that 
$\MapAv$ possesses the geometric properties necessary for mirror 
descent  
analysis: 
it is Lipschitz continuous and its associated Bregman divergence is appropriately 
bounded.
 Since 
$\Map$ is---by definition---an unbiased estimator of $\MapAv$, we 
immediately obtain (thanks to Assumption~\ref{ass:bandit-adversary}) that 
$X_t = \Map[u_t](\eta\sum_{i=1}^{t-1}\g[i])$ satisfies the same regret 
bound in expectation. High-probability bounds follow immediately via  
martingale concentration.

\paragraph{Application to online PCA} As our sketched actions 
are 
of the form $\x 
= x_t x_t^T$, the regret they incur is 
$\lambdamax\big(\sum_{t=1}^T\g\big) 
- \sum_{t=1}^Tx_t^T \g x_t$. Therefore, the vectors $x_t$ can be viewed as 
streaming approximations of the principal component\footnote{For this 
reason we consider gain-maximization rather than loss-minimization, 
which is generally more conventional.} of the cumulative 
matrix $\sum_{i=1}^{t-1}\g[i]$.
This online counterpart of the classical  
principal component analysis problem is the topic of a number of prior 
works \citep[cf.][]{WarmuthKu08,GarberHaMa15,AllenLi17}. 
Our sketch offers regret bounds that are optimal up to constants, with 
computational cost per step as low as any known alternative, and overall 
computational cost better than any in the literature by a factor of at least 
$\log^5 n$~(see Section~\ref{sec:related}). 
Our regret bounds hold for gains $\g$ of any rank or 
sparsity, and our computational scheme (\Cref{sec:expvec-compute}) 
naturally leverages low rank and/or sparsity in the gains.

\paragraph{Application to semidefinite programming (SDP)}
Any feasibility-form SDP is reducible to the matrix saddle-point game 
$\max_{\x[]\in\domain}\min_{y\in\simplex} 
\inner{\sum_{i=1}^m y_i A_i }{\x[]}$, where $\simplex$ is the  
simplex in $\R^m$ and $A_1, \ldots, A_m\in\R^{n\times n}$ are 
symmetric matrices. A simple 
procedure for approximating a saddle-point (Nash equilibrium) for this game is to 
have each player perform online learning, where the max-player observes gains 
$\g = 
\sum_{i=1}^m [y_{t}]_iA_i$ and the min-player observes costs  
$[c_t]_i = 
\inner{A_i}{\x[t]}$. Using standard/matrix multiplicative weights for the 
min/max 
players, respectively, we may produce approximate solutions with additive 
error 
$\epsilon$ in $O(\log(nm)/\epsilon^2)$ iterations, with each iteration costing 
$O(n^3)$ time, due to the MMW computation. In Section~\ref{sec:general-sdp} we 
show that by replacing MMW with our sketch we guarantee $\epsilon$ error 
in a similar number of iterations, but with each iteration costing $\Otil{ 
N/ \sqrt{\epsilon}}$, where $N$ is the problem description size, 
which is often significantly smaller than $n^2$. This guarantee matches the 
state-of-the-art in a number of settings.

\paragraph{Paper outline}
After surveying related work in Section~\ref{sec:related}, we present our main 
contribution in Section~\ref{sec:rank1}: regret 
bounds for our rank-1 randomized projections $\Map$ and their proof via 
the geometry of $\MapAv$. 
In Section~\ref{sec:expvec-compute} we 
describe how to compute $\x$ in $\Otil{\sqrt{\eta t}}$ matrix-vector 
products 
using the Lanczos method.
In Section~\ref{sec:general-sdp} we present in detail the application of our 
sketching 
scheme to semidefinite programming, as described above. We conclude the paper 
in Section~\ref{sec:discussion} by discussing a number of possible 
extensions of our results along with the challenges they present.

\subsection{Related work}\label{sec:related}
MMW appears in a large body of work spanning 
optimization, theoretical computer 
science, and machine 
learning~\citep[e.g.][]{Nemirovski04,WarmuthKu08,AroraHaKa12}. Here, we 
focus on works 
that, like us, attempt to relieve the computational burden of computing the 
 matrix exponential, while preserving the MMW regret guarantees. To our 
knowledge, the first proposal along these lines is due to~\cite{AroraKa07}, 
who apply MMW with a Johnson-Lindenstrauss sketch to semidefinite 
relaxations of combinatorial problems. Subsequent works on positive 
semidefinite programming adopted this 
technique~\citep{PengTaZh16,AllenLeOr16}. To achieve $\eps$-accurate 
solutions, these works require roughly $\eps^{-2}$ matrix exponential 
vector products per mirror projection. 

\cite{BaesBuNe13} apply the accelerated mirror-prox scheme 
of~\cite{Nemirovski04} to 
matrix saddle-point problems and approximate $\MapMMW$ using the 
rank-$k$ sketch~\eqref{eq:rank-k-sketch-def}. Instead of appealing to the 
JL lemma, they absorb the bias and variance of this approximation directly 
into the algorithm's error estimates. This enables a more parsimonious 
choice of $k$; to attain additive error $\eps$, they require 
$k=\Otil{\eps^{-1}}$. See~\Cref{sec:app-general-sdp-comp} for 
additional 
discussion of the performance of this method.

A different line of work, called Follow the Perturbed Leader 
(FTPL)~\citep{KalaiVe05}, 
eschews matrix exponentiation, and instead produces \mbox{rank-1} 
actions $\x=x_t x_t^T$, where $x_t$ is an approximate top eigenvector of 
a random perturbation of $\sum_{i=1}^{t-1}\g[i]$. While a single 
eigenvector 
computation has roughly the same cost as a single matrix-exponential 
vector product, the regret bounds for FTPL---and hence also the total 
work---scale polynomially in the problem dimension $n$: 
\citet{GarberHaMa15} bound the regret by $\Otil{\sqrt{nT}}$ 
and~\citet{DworkTaThZh14} improve the bound to $\Otil{\sqrt{n^{1/2} T}}$  
for gains of rank 1.
In contrast, the regret of MMW and its 
sketches 
depends on $n$ only logarithmically. %

\cite{AllenLi17} give the first fixed-rank sketch with MMW-like regret, proposing a 
scheme called Follow the Compressed Leader (FTCL). Their 
approach is based on replacing the MMW mirror 
projection~\eqref{eq:mmw-def} with the projection corresponding to 
 $\ell_{1-1/q}$ regularization, given by 
$\MapQ(\xM)\defeq (c(\xM)I - \xM)^{-q}$ where $c(\xM)$ is the unique 
$c\in\R$ such that $c I - \xM \succ 0$ and $\tr [(c I - \xM)^{-q}] = 1$. 
They use a sketch of $\MapQ$ similar in spirit 
to~\eqref{eq:rank-k-sketch-def} and prove that $k=3$ suffices to 
obtain 
regret bounds within a polylogarithmic factor of MMW, with $q$ chosen to 
be roughly $\log n$. 

The basis of the FTCL proof strategy is a potential argument used to derive 
regret bounds for the exact $\MapQ$. Their analysis consists of carefully tracing 
this argument, and accounting for the errors caused by sketching in each step of 
the way. In comparison, we believe our analysis is more transparent; rather than 
control multiple series expansion error terms, we establish three simple geometric 
properties of our projection $\MapAv$. We also provide  
tighter bounds; to guarantee $\eps$ average regret, FTCL requires a factor of 
$\Omega(\log^5(n/\eps))$ more online learning steps than our method. 
The 
per-step computational cost of our method is similar to that of FTCL, with 
better polylogarithmic dependence on $n$.
On a practical note, the 
computational scheme we describe in  Section~\ref{sec:expvec-compute} is 
significantly simpler to implement than the one proposed for FTCL. %
\subsection{Notation}
\label{sec:notation}

We use upper case letter for matrices and lower case letters for vectors 
and 
scalars.  
We let $\domainDual$ denote the set of symmetric $n\times n$ 
matrices, 
and let $\domain \defeq \{\x[]\in\domainDual\mid \x[]\succeq 0, \tr \x[] 
= 1\}$ denote the spectrahedron. 
We write $\ones$ for the all-ones vector, and let
 $\simplex[n]\defeq \{x\in\R^n \mid x \ge 0, \ones^T x = 1\}$
 denote the simplex. 
We let $\inner{\xM}{\x[]}=\tr (\xM^T \x[] 
)$ 
denote the Frobenius inner product between matrices.
For $\x[] \in \domainDual$, we let 
$\lambdamax(\x[])=\lambda_1(\x[])\ge \lambda_2(\x[]) \ge \ldots \ge 
\lambda_n(\x[])=\lambdamin(\x[])$ 
denote the eigenvalues of $\x[]$ sorted in descending order. 
For $x\in\R^n$ and $p\ge 1$ we let $\norm{x}_p = 
\big(\sum_{i=1}^n \left| x_i \right|^p\big)^{1/p}$ denote the $\ell_p$ 
norm,
and for $\x[]\in\domainDual$, we let $\norm{\x[]}_p \defeq 
\norm{\lambda(\x[])}_p$ be the standard Schatten $p$-norm. In particular, 
$\linf{\x[]}=\max\{\lambdamax(\x[]), 
-\lambdamin(\x[])\}$ is the Euclidean operator norm and $\lone{\x[]} 
= \sum_{i=1}^n |\lambda_i (\x[])|$ is the nuclear norm. 
We write $\uniform(\sphere^{n-1})$ for the uniform distribution 
over the unit sphere in $\R^n$.
\section{A rank-1 sketch of matrix multiplicative  weights}
\label{sec:rank1}

In this section, we state and prove our main result: regret bounds for a
rank-1 sketch of the matrix multiplicative weights method. Let us recall our
sketch. At time step $t$, having observed gain matrices $\g[1],\ldots,
\g[t-1]\in\domainDual$, we independently draw\footnote{%
Since $\Map$ is invariant to scaling of $u$, it has the same distribution for 
$u$ standard Gaussian or uniform on a sphere.}
 $u_t \sim
\uniform(\sphere^{n-1})$ and play the rank-1 matrix
\begin{equation}\label{eq:rand-da}
  \x \defeq \Map[u_t]\left(\eta \sum_{i=1}^{t-1} \g[i] \right),
  ~\mbox{where}~
  \Map(\xM) \defeq \frac{ e^{\xM/2}uu^T e^{\xM/2}}{u^T e^{\xM} u} 
  = \frac{vv^T}{v^T v} \mbox{ for }v=e^{\xM/2}u.
\end{equation} 
We call $\Map:\domainDual\to\domain$ the 
\emph{randomized mirror 
  projection}. The key computational consideration is that we can
evaluate $\Map(\xM)$ efficiently, while
on the analytic side, we show
that the update~\eqref{eq:rand-da}
defines \emph{on average} an efficient mirror descent 
procedure. The regret bounds for $X_t$ then follow.
\subsection{Expected regret bounds}\label{sec:expected-regret}

The focus of our analysis is the
\emph{average mirror projection}
\begin{equation}\label{eq:proj-av-def}
  \MapAv(\xM) \defeq \E_u \Map[u](\xM)
  ~\mbox{ and action sequence }~
  \xAv \defeq \MapAv\left(\eta\sum_{i=1}^{t-1} \g[i] \right),
\end{equation}
where $\E_u$ denotes expectation w.r.t.\ to $u \sim 
\uniform(\sphere^{n-1})$. As we show in Section~\ref{sec:proj-geometry} 
to come,
$\MapAv$ is the gradient of the function
\begin{equation*}
  \mapAv(\xM) \defeq \E_u \log \left(\inner{e^{\xM}}{uu^T} \right) = 
  \E_u \log \left(u^T e^\xM u \right),
\end{equation*}
which we also show\footnote{For \emph{fixed} $u\in\R^n$, 
however, 
$\Map \ne \grad \log \left(u^T e^\xM u \right)$ and we do not know 
if it is the gradient of any other function. Moreover, $\xM\mapsto \log 
\left(u^T 
e^\xM u \right)$ is not convex.} is a convex spectral 
function~\citep{Lewis96}.  As a consequence, we can write the 
average action 
$\xAv$ in the familiar dual averaging~\citep{Nesterov09} or Follow the
Regularized Leader~\citep[e.g.][Ch.~5]{Hazan16} form
\begin{equation*}
  \xAv = \argmax_{\x[] \in \domain}\left\{\eta 
  \sum_{i=1}^{t-1}\inner{\g[i]}{\x[]}
  - \mapConj(\x[])\right\}
\end{equation*} 
where $\mapConj(\x[]) = \sup_{\xM\in\domainDual}\left\{ 
\inner{\xM}{\x[]} - \mapAv(\xM)\right\}$ is the convex conjugate of 
$\mapAv$. In this standard approach, the regularizer $\mapConj$ defines 
the scheme, and regret analysis proceeds by showing that  $\mapConj$ is 
strongly convex and has bounded range. The former property is equivalent 
to the smoothness of $\mapAv$.

In contrast, our starting point is the definition~\eqref{eq:proj-av-def} of 
the projection $\MapAv$, and 
we 
find it more convenient to argue about $\MapAv$ and $\mapAv$ directly. 
Toward 
that end, for any 
$\xM,\xM'\in\domainDual$ we let
\begin{equation}\label{eq:bregman-def}
  \breg{\xM}{\xM'} \defeq \mapAv(\xM') - \mapAv(\xM) - 
  \inner{\xM'-\xM}{\MapAv(\xM)}
\end{equation}
denote the Bregman divergence induced by $\mapAv$.  We show
that $\breg{\xM}{\cdot}$ has the properties---analogous to those
arising from duality in analyses of dual
averaging~\citep{Nesterov09}---necessary to establish our regret 
bounds.

\begin{proposition}\label{prop:regret-properties}
  The projection $\MapAv$ and divergence $\bregBlank$ satisfy
  \begin{enumerate}[label=\propenum*., ref=\propenum*]
  \item\label{item:smoothness} Smoothness: for every $\xM, 
    \dM\in\domainDual$, 
    $\breg{\xM}{\xM + \dM} 
    \le \frac{3}{2} \linf{\dM}^2$.
  \item\label{item:refined-smoothness}
    Refined smoothness for positive shifts: for every $\xM, 
    \dM \in\domainDual$ such that $\dM \succeq 0$ and 
    $\linf{\dM} \le \frac{1}{6}$,
    $
      \breg{\xM}{\xM + \dM} \le 3\linf{\dM} 
      \inner{\dM}{\MapAv(\xM)}.
      $
  \item\label{item:diameter} Diameter bound: for every  
    $\xM,\xM'\in\domainDual$, 
    $\breg{\xM}{0} - \breg{\xM}{\xM'}  \le \log(4n)$.
  \item\label{item:invert} Surjectivity: for every $\x[] \in 
    \interior\domain$ 
    there exists $\xM\in\domainDual$ such that $\MapAv(\xM) = \x[]$.
  \end{enumerate} 
\end{proposition}
We return to Proposition~\ref{prop:regret-properties} and prove it in 
Section~\ref{sec:proj-geometry}. The proposition gives the following 
regret 
bounds for the averaged actions $\xAv$.
\begin{restatable}{theorem}{restateAvRegret}\label{thm:av-regret}
  Let $\g[1],\ldots,\g[T]$ be \emph{any} sequence of gain matrices in
  $\domainDual$ and let $\xAv = \MapAv(\eta \sum_{i=1}^{t-1}\g[i])$ as in
  Eq.~\eqref{eq:proj-av-def}. Then, for every $T\in\N$,
  \begin{equation}
    \label{eq:regret-bound-basic}
    \lambdamax\left( \sum_{t=1}^T \g \right) - \sum_{t=1}^T 
    \inner{\g}{\xAv}  \le 
    \frac{\log(4n)}{\eta} + 
    \frac{3\eta}{2} \cdot 
    \sum_{t=1}^T 
    \linf{\g}^2.
  \end{equation}
  If additionally $0 \preceq \g \preceq I$ for every $t$ and $\eta \le 
  \frac{1}{6}$,
  \begin{equation}\label{eq:regret-bound-refined}
    \lambdamax\left( \sum_{t=1}^T \g \right) - \sum_{t=1}^T 
    \inner{\g}{\xAv} 
    \le  \frac{\log{(4n)}}{\eta} +
    3\eta \cdot  \lambdamax\left( 
    \sum_{t=1}^T \g \right).
  \end{equation}
\end{restatable}
\noindent

We prove Theorem~\ref{thm:av-regret} in
Appendix~\ref{sec:app-da-regret-proof}. The proof is essentially the
standard dual averaging telescoping argument~\citep{Nesterov09}, which 
we perform using only the properties in
Proposition~\ref{prop:regret-properties}. Indeed, matrix multiplicative 
weights satisfies a version of 
Proposition~\ref{prop:regret-properties} with slightly smaller constant 
factors, 
and its regret bounds follow similarly.

The projection $\MapAv$ is no easier to compute than the matrix 
multiplicative 
weights projection. However, $\Map$ is easily computed and is unbiased  
for 
$\MapAv$. Consequently---under 
Assumption~\ref{ass:bandit-adversary}---the sketch $\Map$ inherits the 
regret guarantees in Theorem~\ref{thm:av-regret}. To argue this formally, 
we 
define the
$\sigma$-fields
\begin{equation*}
  \Filt[t] \defeq \sigma(\g[1], \x[1],
  \ldots, \g[t] \x[t], \g[t+1]),
\end{equation*}
so that $\g[t] \in \Filt[t-1]$ and $\xAv[t]
\in \Filt[t-1]$, while, under 
Assumption~\ref{ass:bandit-adversary}, $\E[\x[t] \mid \Filt[t-1]]
= \xAv[t]$ because $u_t \sim \uniform(\sphere^{n-1})$, independent
of $\Filt[t-1]$. Consequently, we 
have the following

\begin{corollary}\label{cor:expected-regret}
  Let $\g[1],\ldots,\g[T]$ be symmetric gain matrices satisfying 
  Assumption~\ref{ass:bandit-adversary} and let $\x$ be generated 
  according to 
  Eq.~\eqref{eq:rand-da}. Then
  \begin{equation*}
    \E\bigg[\lambdamax\bigg(\sum_{t=1}^T \g \bigg) - \sum_{t=1}^T 
      \inner{\g}{\x}\bigg]
    \le 
    \frac{\log(4n)}{\eta} + 
    \frac{3\eta}{2} \cdot 
    \sum_{t=1}^T 
    \E\big[
      \linf{\g}^2 \big].
  \end{equation*}
  If additionally $0 \preceq \g \preceq I$ for every $t$ and $\eta \le 
 \frac{1}{6}$,
  \begin{equation*}
    \E\bigg[\lambdamax\bigg(\sum_{t=1}^T \g \bigg) - \sum_{t=1}^T 
      \inner{\g}{\x} \bigg]
    \le 
    \frac{\log{(4n)}}{\eta}
    +
    3\eta \cdot
    \E\bigg[\lambdamax\bigg( 
    \sum_{t=1}^T \g \bigg)\bigg].
  \end{equation*}
\end{corollary}
\begin{proof}
	Using $\g[t] \in \Filt[t-1]$ and $\E[\x[t] \mid \Filt[t-1]]
	= \xAv[t]$, we have $\E\inner{\g}{\x}  = \E\left[\E[\inner{\g}{\x} \mid 
	\Filt[t-1]]\right]
	= \E\inner{\g}{\xAv}$, 
  and so the result is immediate from taking expectation in 
  Theorem~\ref{thm:av-regret}.
\end{proof}

It is instructive to compare these guarantees
to those for the full (non-approximate) matrix multiplicative weights
algorithm.
Let
\begin{equation*}
  \regret \defeq\E \bigg[ \lambdamax\bigg(  \frac{1}{T}\sum_{t=1}^T 
    \g \bigg) -  \frac{1}{T}\sum_{t=1}^T \inner{\g}{\x}\bigg]
\end{equation*}
denote the 
expected \emph{average} regret at time $T$.
If $\linf{\g} \le 
1$ for every $t$, the 
bound~\eqref{eq:regret-bound-basic} along with 
Corollary~\ref{cor:expected-regret} imply, for $\eta = 
({2\log(4n)/(3T)})^{1/2}$,
\begin{equation*}
  \regret \le 
  \sqrt{\frac{6\log(4n)}{T}},
  ~~ \mbox{i.e.} ~~
  \regret \le \eps
  ~~ \mbox{for}~
  T \ge \frac{6\log(4n)}{\eps^2}.
\end{equation*} 
In contrast, the matrix multiplicative weights
procedure~\eqref{eq:mmw-def} guarantees average regret 
below $\eps$ in $2\log(n)/\eps^2$ steps, so our guarantee is worse by a 
factor of roughly $3$. 

The bound~\eqref{eq:regret-bound-refined} guarantees smaller 
\emph{relative} average regret when we additionally assume $0 \preceq \g 
\preceq I$ for every $t$ and an a-priori upper bound of the form 
$\lambda^\star \defeq \lambdamax(\frac{1}{T}\sum_{t=1}^T \g ) \ge 
\lambda_0$. 
Here, a judicious choice of $\eta$ guarantees  $\regret/\lambda^\star \le 
\varepsilon$  for $T \ge 
12\log(4n)/(\lambda_0 \varepsilon^2 )$.
Again, this is slower than the 
corresponding guarantee for matrix multiplicative
weights by a factor of roughly 3. Relative regret bounds of the 
form~\eqref{eq:regret-bound-refined} are useful in several application of 
multiplicative weights and its matrix variant~\citep{AroraHaKa12}, e.g.\ 
width-independent solvers for linear and positive semidefinite 
programs~\citep{PengTaZh16}.
\subsection{High-probability regret bounds}\label{sec:high-prob}
\newcommand{\mg}{\mathfrak{m}}

Using standard martingale convergence arguments 
 \citep[cf.][]{CesaBianchiCoGe04, NemirovskiJuLaSh09},
we can provide high-probability convergence guarantees for
our algorithm. Indeed, we have already observed in
Corollary~\ref{cor:expected-regret} that 
$\E\left[\inner{\g}{\x}\mid\Filt\right]=\inner{\g}{\xAv}$ and therefore 
$\inner{\g}{\x-\xAv}$ is a martingale difference
sequence adapted to the filtration
$\Filt[t]$.
As $\left|\inner{\g}{\x}\right|\le \linf{\g}\lone{\x} = \linf{\g}$, 
the 
martingale 
$\sum_{i=1}^t\inner{\g[i]}{\x[i]-\xAv[i]}$
has bounded differences whenever $\linf{\g}$ is 
bounded, so that the next theorem is an immediate
consequence of the Azuma-Hoeffding inequality and its multiplicative 
variant~\cite[][Lemma G.1]{AllenLi17}

\begin{restatable}{corollary}{restateHighProb}\label{cor:high-probability-regret}
  Let $\g[1],\ldots,\g[T]$ be symmetric gain matrices satisfying 
  Assumption~\ref{ass:bandit-adversary} and let $\x$ be generated 
  according to 
  Eq.~\eqref{eq:rand-da}. If $\linf{\g}\le 
  1$ for every $t$, then for every 
  $T\in\N$ 
  and $\delta\in(0,1)$, with probability
  at least $1 - \delta$,
  \begin{equation}
    \label{eq:regret-bound-basic-high-prob}
    \lambdamax\left( \sum_{i=1}^T \g \right) - \sum_{t=1}^T 
    \inner{\g}{\x} \le 
    \frac{\log(4n)}{\eta} + 
    \frac{3\eta}{2}\,T + \sqrt{2T\log{\tfrac{1}{\delta}}}.
  \end{equation}
  If additionally $0 \preceq \g \preceq I$ for every $t$ and 
  $\eta \le \frac{1}{6}$, then
  with probability at least $1 - \delta$,
  \begin{equation}\label{eq:regret-bound-refined-high-prob}
    \lambdamax\left( \sum_{i=1}^T \g \right) - \sum_{t=1}^T 
    \inner{\g}{\x} 
    \le 
    \frac{\log{(4n/\delta)}}{\eta} + 4\eta \, \lambdamax\left( 
    \sum_{i=1}^T \g \right).
  \end{equation}
\end{restatable}
\noindent
We give the proof of Corollary~\ref{cor:high-probability-regret} in 
Appendix~\ref{sec:app-regret-hp}. 

Our development uses Assumption~\ref{ass:bandit-adversary} only through 
its consequence $\E[\x[t] \mid \Filt[t-1]]
= \xAv[t]$. Therefore, our results apply to any adversary that produces 
gains with such martingale structure, a weaker requirement than 
Assumption~\ref{ass:bandit-adversary}.
\subsection{Analyzing the average mirror 
projection}\label{sec:proj-geometry}

In this section we outline the proof of
Proposition~\ref{prop:regret-properties}, which constitutes the core
technical contribution of our paper. Our general strategy is to relate the
average mirror projection to the multiplicative weights projection, which satisfies  a 
version of
Proposition~\ref{prop:regret-properties}. Our principal mathematical tool 
is the theory of convex, twice-differentiable spectral 
functions~\citep{Lewis96,LewisSe01}.

We begin with the vector log-sum-exp, or softmax, function
\begin{equation*}
  \softmax(v) \defeq \log\bigg(\sum_{j = 1}^n e^{v_j}\bigg)
  ~~\mbox{and its gradient} ~~
  \gradsoftmax(v) 
  = \frac{e^v}{\ones^T e^v}, %
\end{equation*}
where we write $e^v$ for $\exp(\cdot)$ applied elementwise to $v$ and 
$\ones$ 
for the all-ones vector. Note that $\gradsoftmax:\R^n\to\simplex[n]$ is 
the mirror projection associated with (vector) multiplicative weights. Let 
$\xM\in\domainDual$ have eigen-decomposition 
$Y=Q\diag(\lambda)Q^T$. The matrix softmax function is
\begin{equation*}
\mapMMW(\xM)\defeq \log \tr e^{\xM} = \softmax(\lambda)
~~\mbox{and} ~~
\MapMMW(\xM) = \grad \mapMMW(\xM) = \frac{e^{\xM}}{\tr e^{\xM}} = 
Q\diag(\gradsoftmax(\lambda))Q^T
\end{equation*}
is the matrix multiplicative weights mirror projection. 

We now connect the function 
$\mapAv(\xM) = \E_u[\log \tr(e^{\xM} uu^T)]$ and 
the projection
$\MapAv(\xM)
= \E_u 
\frac{e^{\xM/2}uu^T e^{\xM/2}}{u^Te^\xM u}$ to their counterparts 
$\mapMMW, \MapMMW$ and $\softmax$.
\begin{lemma}
  \label{lem:w-characterization}
  Let $\xM \in \domainDual$ have eigen-decomposition $\xM= 
  Q\diag(\lambda) Q^T$. Let $w \in \simplex[n]$ be drawn from a
  $\dirichlet$ distribution. Then
  \begin{equation}
    \label{eq:softmax-av-identity}
    \mapAv(\xM) = \E_w\left[\softmax(\lambda + 
    \log w)\right] = \E_w \mapMMW(Y + Q\diag(\log w)Q^T )
  \end{equation}
  where $\log$ is applied elementwise. The function $\mapAv$ is convex 
  and its gradient is
  \begin{equation}
    \label{eq:proj-av-identity}
    \MapAv(\xM) = \grad \mapAv(\xM) = 
     Q \diag\left(\E_w[\grad\softmax(\lambda
    + \log w)]\right) Q^T
    = \E_w \MapMMW(Y+ Q \diag( \log w)Q^T ).
  \end{equation}
\end{lemma}
\begin{proof}
  Let $u$ be uniformly distributed over the unit sphere in $\R^n$ and note 
  that $u$ and $Q^T u$ are identically distributed. Therefore, for $\Lambda 
  = \diag(\lambda)$,
  \begin{equation*}
    \mapAv(\xM) = \E_u \log \left( u^T e^\xM u\right )
    = \E_u \log \left( (Q^Tu)^T e^\Lambda (Q^Tu)\right )
    = \E_u \log \left( u^T e^\Lambda u\right )
    = \mapAv(\Lambda).
  \end{equation*}
  Further,  a vector $w$ with coordinates\footnote{The letter $w$ 
  naturally denotes a vector of `weights' in the simplex. Here, it is also 
  double-$u$.} $w_i = u_i^2$ 
  has a 
  $\dirichlet$ distribution.
  Hence,
  \begin{equation*}
    \mapAv(\Lambda) = \E_u \log \Bigg( \sum_{i=1}^n u_i^2 e^{\lambda_i} 
    \Bigg)
    = \E_w \log \Bigg( \sum_{i=1}^n e^{\lambda_i + \log w_i} \Bigg)
    = \E_w \softmax(\lambda + \log w),
  \end{equation*}
  establishing the identity~\eqref{eq:softmax-av-identity}.
  
  Evidently, $\mapAv(\xM)$ is a spectral function---a 
  permutation-invariant function of the eigenvalues of $\xM$. Moreover, 
  since $\softmax$ is convex, $\lambda\mapsto\E_w \softmax(\lambda + 
  \log w)$ is also convex, and~\citet[][Corollary 2.4]{Lewis96} shows that 
  $\mapAv$ is convex. Moreover,~\citet[][Corollary 3.2]{Lewis96} gives
  \begin{flalign*}
    \grad\mapAv(\xM)
    = Q \diag(\grad \E_w[\softmax(\lambda + \log w)]) Q^T
    = \E_w\MapMMW(Y+Q\log(w)Q^T).
  \end{flalign*}
  It remains to show that $\MapAv(\xM) = \grad\mapAv(\xM)$.
  Here we again use the rotational symmetry 
  of $u$ to write
  \begin{equation*}
    \MapAv(\xM) = \E_u  \frac{e^{\xM/2} uu^T e^{\xM/2}}{u^T 
      e^\xM u}
    =Q \, \left( \E_u  \frac{e^{\Lambda/2} (Q^Tu)(Q^Tu)^T 
      e^{\Lambda/2}}{(Q^Tu)^T 
      e^\Lambda (Q^Tu)}\right ) Q^T = Q\,\MapAv(\Lambda)Q^T.
  \end{equation*}
  Moreover,
  \begin{equation*}
    \MapAv(\Lambda)_{ij} = \E_u 
    \frac{u_i u_je^{(\lambda_i+\lambda_j)/2}}{\sum_{k=1}^n u_k^2 
      e^{\lambda_k}} 
    \stackrel{(\star)}{=} \E_u 
    \frac{u_i^2 e^{\lambda_i}\indic{i=j}}{\sum_{k=1}^n u_k^2 
      e^{\lambda_k}} 
    = \E_w \grad_i \softmax(\lambda + \log w)
    \indic{i = j}
  \end{equation*}
  where the equality $(\star)$ above follows because $u_i$ has a
  symmetric distribution, even conditional on $u_j, j \neq i$,
  so $\E \left[ u_i u_j \mid u_1^2, \ldots, 
    u_n^2, u_j \right] = 0$ for $i\ne j$. 
\end{proof}

Lemma~\ref{lem:w-characterization} is all we need in order to prove 
parts~\ref{item:diameter} and~\ref{item:invert} of 
Proposition~\ref{prop:regret-properties}.

\begin{proof}\textbf{(Proposition~\ref{prop:regret-properties}, parts
    \ref{item:diameter} and \ref{item:invert})}
  We first observe the following simple lower bound on $\mapAv$, 
  immediate 
  from identity~\eqref{eq:softmax-av-identity} in 
  Lemma~\ref{lem:w-characterization},
  \begin{equation}\label{eq:map-av-bound}
    \mapAv(\xM) = \E_w \log\Big(\sum_{i=1}^n e^{\lambda_i(\xM) 
      + \log w_i}\Big)
    \ge \lambdamax(\xM) + E_{w_1} \log w_1
    \ge \lambdamax(\xM) - \log (4n),
  \end{equation}
  where $ \E_{w_1} \log w_1 \ge - \log (4n)$ comes from
  noting that $w_1 \sim \betadist(\frac{1}{2},\frac{n-1}{2})$ (see
  Lemma~\ref{lem:beta-log-expectation} in
  Appendix~\ref{sec:app-beta-facts}). 
  For matrices 
  $\xM\in\domainDual$ and $\x[]\in\domain$,
  \begin{equation*}
  \inner{\xM}{\x[]} = \inner{\xM-\lambdamin(\xM)I}{\x[]} + 
  \lambdamin(\xM) \tr \x[]
  \le 
  \linf{\xM-\lambdamin(\xM)I}\lone{{\x[]}} + \lambdamin(\xM) \tr \x[]
  = \lambdamax(\xM),
  \end{equation*}
  where the final equality is due to $\linf{\xM-\lambdamin(\xM)I} = 
  \lambdamax(\xM)-\lambdamin(\xM)$ for every $\xM\in\domainDual$ 
  and $\lone{\x[]} = \tr \x[] = 1$ for every $\x[] \in\domain$.
  Combining this bound with~\eqref{eq:map-av-bound}, we have that
   \begin{equation}\label{eq:dgf-av-bound}
  \inner{\xM}{\x[]} - \mapAv(\xM) \le \log(4n)
  \end{equation}
  for every $\xM\in\domainDual$ and $\x[]\in\domain$. 
  Part~\ref{item:diameter} follows since
  \begin{equation*}
    \breg{\xM}{0}-\breg{\xM}{\xM'}
    =
    \mapAv(0)+ \inner{\xM'}{\MapAv(\xM)}-\mapAv(\xM')
    \le \mapAv(0) + \log (4n) = \log (4n),
  \end{equation*}
  where we used the bound~\eqref{eq:dgf-av-bound} with 
  $\x[]=\MapAv(\xM)$ and the fact that $\mapAv(0)=\E_w \log(\ones^T w)=0$.

	To show Part~\ref{item:invert}, let $\mapConj(\x[]) \defeq \sup_{\xM
          \in \domainDual}\{\inner{\xM}{\x[]} - \mapAv(\xM)\}$ be the convex
        conjugate of $\mapAv$. Eq.~\eqref{eq:dgf-av-bound} implies that
        $\mapConj(\x[]) < \infty$ for all $\x[]\in\domain$, and therefore
        $\interior \domain \subseteq \interior \;\dom \mapConj$.  Every 
        convex
        function has nonempty subdifferential on the relative interior of
        its domain~\citep[Theorem~X.1.4.2]{HiriartUrrutyLe93ab},
        and thus for   
        $X\in \interior \domain$ there exists $\xM \in \partial 
	\mapConj(\x[])$. By definition of $\mapConj$, any such $\xM$ 
	satisfies $\x[] = \grad\mapAv(\xM) = \MapAv(\xM)$, as required.
\end{proof}

Proving parts~\ref{item:smoothness} and~\ref{item:refined-smoothness}
requires second order information on $\mapAv$. For twice differentiable 
function $f$, we denote $\hess f(A)[B,B] = \frac{\partial^2}{\partial t^2}
f(A + tB)|_{t = 0}$. It is easy to verify that, for every 
$\lambda,\delta\in\R^n$,  $$\delta^T \hess \softmax(\lambda) \delta = 
\hess 
\softmax(\lambda)[\delta,\delta] \le (\delta^2)^T 
\gradsoftmax(\lambda),$$
where $[\delta^2]_i = \delta_i^2$; this concisely captures the pertinent  
second order structure of the multiplicative weights mirror projection. 
\citet{Nesterov07b} 
shows that this property extends to the matrix case.

\begin{restatable}{lem}{restateMmwSpectralBound}
  \label{lem:mmw-spectral-bound}
  For any $\xM, \dM\in\domainDual$, $\hess \mapMMW(\xM)[\dM, 
    \dM] 
    \le 
    \inner{\dM^2}{ 
    \MapMMW(\xM)}$.
\end{restatable}
\noindent
In Appendix~\ref{sec:app-mmw-spec-bound-proof} we explain how to find 
this result in~\citet{Nesterov07b}, as 
it is not explicit there. 
In view of Lemma~\ref{lem:w-characterization}, it is natural to hope that 
$\hess \mapAv$ and $\hess \mapMMW$ are also related via simple 
expectation. Unfortunately, this fails; 
we can, however, derive a bound.

\begin{restatable}{lem}{restateMainSpecBound}\label{lem:main-spec-bound}
  For any $\xM, \dM\in\domainDual$, orthogonal eigenbasis $Q$ 
  for 
  $Y$, and $w\sim\dirichlet$,
   \begin{flalign}
  \label{eq:main-spec-bound1}
  \hess \mapAv (\xM)[\dM, \dM] &\le 
  3 \cdot  \E_w \hess 
  \mapMMW(\xM + 
  Q\,\diag(\log w)Q^T)[\dM, \dM]
  	\\ & \le
  	3 \inner{\dM^2}{\MapAv(\xM)}.
  	\label{eq:main-spec-bound2}
  \end{flalign}
\end{restatable}
\arxiv{
Our proof of Lemma~\ref{lem:main-spec-bound} is technical; we sketch it 
here briefly and give it in full 
Appendix~\ref{sec:app-main-spec-bound-proof}. The key ingredient in the 
proof is a formula for the Hessian of spectral functions~\citep{LewisSe01}. 
Using the spectral 
characterization~\eqref{eq:softmax-av-identity}, the formula gives that
\begin{equation*}
\hess \mapAv(\xM)[\dM, \dM] = 
\diag(\tilde\dM)^T \big[ \E_w\hess \softmax(\lambda + \log w) \big]
\diag(\tilde\dM)
+ 
\inner{\E_w A^w(\lambda)}{\tilde{\dM} \circ \tilde{\dM}}.
\end{equation*}
where $\tilde{D}=Q^T D Q$, $\diag(\tilde{D})\in\R^n$ is the vector 
containing the 
diagonal entries of $\tilde{D}$, $A\circ B$ denotes elementwise 
multiplication of $A$ and $B$,  
and $A_{ij}^w(\lambda) \defeq \frac{\grad_i \softmax(\lambda+\log(w)) - 
	\grad_j \softmax(\lambda+\log(w))}{\lambda_i -\lambda_j}\indic{i\ne 
	j}$. 
With the shorthand $\xM_{\{w\}}\defeq  \xM + Q\diag(\log w)Q^T$, we  
use the formula of~\citet{LewisSe01} again to express 
$\hess\mapMMW(\xM_{\{w\}})$ as 
\begin{equation*}
\hess \mapMMW(\xM_{\{w\}})[\dM, \dM] = 
\diag(\tilde\dM)^T \big[\hess \softmax(\lambda + \log w) \big]
\diag(\tilde\dM)
+ 
\inner{ A^\ones(\lambda+\log w)}{\tilde{\dM} \circ \tilde{\dM}},
\end{equation*}
where $A^{\ones}=A^{\tilde{w}}$ evaluated at $\tilde{w} = 
\ones$.
The bulk of the proof is dedicated to establishing the entry-wise bounds 
\begin{equation*}
\E_w A_{ij}^w(\lambda) \le 
\E_w \left[\left( 1+ 
\frac{\tanh\big(\frac{\lambda_i-\lambda_j}{2}\big)\big| \log  
	\frac{w_i}{w_j} \big|}{\lambda_i - \lambda_j}
\right)
A^{\ones}_{ij}(\lambda + \log w)\right]
\le 3\cdot \E_w A^{\ones}_{ij}(\lambda + \log w).
\end{equation*}
The first inequality follows from pointwise analysis of a symmetrized 
version of $A_{ij}^{w}$. The second inequality follows from piecewise 
monotonicity of $A_{ij}^{w}$ as a function of 
$\log\frac{w_i}{w_j}\sim\mathrm{logit}\,\mathrm{Beta}(\half, \half)$, 
combined with tight exponential tail bounds for the latter. Substituting the 
bound on $\E_w A_{ij}^w(\lambda)$ into the expression for $\hess 
\mapAv(\xM)$ and comparing with $\E_w \hess 
\mapMMW(\xM_{\{w\}})$ yields the desired 
result~\eqref{eq:main-spec-bound1}. 
Applying Lemma~\ref{lem:mmw-spectral-bound} and recalling the 
identity~\eqref{eq:proj-av-identity} yields
\begin{equation*}
\E_w \hess 
\mapMMW(\xM + 
Q\,\diag(\log w)Q^T)[\dM, \dM]
\le \inner{\dM^2}{\E_w \MapMMW(Y+Q\diag(\log w)Q^T)}
=  \inner{\dM^2}{\MapAv(Y)},
\end{equation*}
establishing the final bound~\eqref{eq:main-spec-bound2}.

The bound~\eqref{eq:main-spec-bound2} gives the remaining parts of 
Proposition~\ref{prop:regret-properties}.

\begin{proof}
\textbf{(Proposition~\ref{prop:regret-properties}, parts
		\ref{item:smoothness} and \ref{item:refined-smoothness})}
	Fix $\xM,\dM\in\domainDual$ and let $p(t)\defeq\mapAv(\xM+t\dM)$. 
	The 
	Bregman divergence~\eqref{eq:bregman-def} admits the integral 
	form 
	\begin{flalign}\label{eq:bregman-hess-form}
	\breg{\xM}{\xM + \dM} 
	&= p(1)-p(0)-p'(0) 
	= \int_{0}^1 {(p'(t)-p'(0))}dt 
	= \int_{0}^1 \int_{0}^t {p''(\tau)} d\tau dt 
	\nonumber \\ &
	=
	\int_{0}^{1}\int_{0}^{t}
	{\hess \mapAv(\xM + \tau \dM)[\dM, \dM]} d\tau dt.
	\end{flalign}
	Note that since $\MapAv(\xM)\in\domain$ for every 
	$\xM\in\domainDual$, $\inner{\dM^2}{\MapAv(\xM)} \le 
	\linfs{\dM^2} 
	\lones{\MapAv(\xM)} = \linfs{\dM}^2$. Therefore, the 
	bound~\eqref{eq:main-spec-bound2}  gives
	\begin{equation*}
	\hess \mapAv(\xM + \tau \dM)[\dM, \dM] \le 
	3\linf{\dM}^2.
	\end{equation*}
	Substituting back into~\eqref{eq:bregman-hess-form} and using 
	$\int_0^1\int_0^t d\tau dt = \half$ gives 
	Proposition~\ref{prop:regret-properties}.\ref{item:smoothness}.

	When $\dM \succeq 0$, we have
	\begin{equation*}
	\inner{\dM^2}{\MapAv(\xM)} =  
	\inner{\dM}{\dM^{1/2}\MapAv(\xM)\dM^{1/2}} \le  \linfs{\dM} 
	\lones{\dM^{1/2}\MapAv(\xM)\dM^{1/2}} = \linfs{\dM} 
	\inner{\dM}{\grad \mapAv(\xM)}.
	\end{equation*}
	Plugging the bound above into the bound~\eqref{eq:main-spec-bound2} 
	and substituting back into~\eqref{eq:bregman-hess-form} gives
	\begin{equation}\label{eq:breg-pos-shift-intermediate-bound}
	\breg{\xM}{\xM + \dM} \le 3\linf{\dM} 
	\int_{0}^{1}\int_{0}^{t}
	{\inner{\dM}{\grad \mapAv(\xM + \tau \dM)}} d\tau dt.
	\end{equation}
	Moreover,
	\begin{flalign}
	\label{eq:breg-inner-integral}
	\int_{0}^{t}
	{\inner{\dM}{\grad \mapAv(\xM + \tau \dM)}} d\tau
	&=
	\int_{0}^{t}
	{p'(\tau)} d\tau
	= p(t)-p(0) 
	= \breg{\xM}{\xM + t\dM} + \inner{t\dM}{\MapAv(\xM)},
	\end{flalign}
	where the final equality uses the definition~\eqref{eq:bregman-def} of 
	the Bregman divergence. Note also that $v(t) \defeq 
	\breg{\xM}{\xM + 
		t\dM}$ is increasing for $t\ge0$ due to convexity of $\mapAv$; 
	$t v'(t) = 
	\inner{t \dM}{\grad\mapAv(\xM + t\dM) - \grad\mapAv(\xM)} \ge 0$. 
	Therefore, 
	the equality~\eqref{eq:breg-inner-integral} implies $\int_{0}^{t}
	{\inner{\dM}{\grad \mapAv(\xM + \tau \dM)}} d\tau \le 
	\breg{\xM}{\xM 
		+ 
		\dM} + t\cdot \inner{\dM}{\MapAv(\xM)}$ for every $0\le t\le 1$. 
	Substituting this back 
	into~\eqref{eq:breg-pos-shift-intermediate-bound} 
	and rearranging gives
	\begin{equation*}
	\Big(1 - 3\linf{\dM}  \Big)
	\breg{\xM}{\xM 
		+ 
		\dM}
	\le \frac{3}{2}\linf{\dM}  
	\inner{\dM}{\MapAv(\xM)}.
	\end{equation*}
	establishing part~\ref{item:refined-smoothness} of the proposition, as 
	$1 - 
	3\linf{\dM}\ge \half$ by assumption.
\end{proof}
}
\notarxiv{
Our proof of Lemma~\ref{lem:main-spec-bound} is technical; we give it in 
Appendix~\ref{sec:app-main-spec-bound-proof}. The key ingredient 
in the proof is a formula for the Hessian of spectral
functions~\citep{LewisSe01}. The remainder of 
Proposition~\ref{prop:regret-properties} follow from the 
bound~\eqref{eq:main-spec-bound2} via $\breg{\xM}{\xM + 
\dM} = \int_{0}^{1}\int_{0}^{t}{\hess \mapAv(\xM + \tau \dM)[\dM, \dM]} 
d\tau dt$; we give the details in \Cref{sec:app-prop11-proof}.
}%

\section{Efficient computation of matrix exponential-vector  products}
\label{sec:expvec-compute}
The main burden in computing the randomized mirror 
projections~\eqref{eq:rand-da} lies in computing $e^{A}b$ for 
$A\in\domainDual$ and $b\in\R^n$. Matrix exponential-vector products 
have widespread use in solutions of differential 
equations~\citep[cf.][]{Saad92,HochbruckOs10}, and also appear as core 
components in a number of theoretical algorithms~\citep{AroraKa07, 
OrecchiaSaVi12, JambulapatiKiSi18}. Following a large body of literature 
\citep[cf.][]{MolerLo03}, we  
approximate $e^{A}b$ via the classic Lanczos method \citep{Lanczos50}, an 
iterative process for 
computing $f(A)b$ for general real functions $f$ applied to matrix $A$. 
The Lanczos approximation enjoys strong convergence guarantees upon 
which we base our analysis~\citep{SachdevaVi14}. It is also eminently 
practical: 
the 
only 
tunable parameter is the number of iterations, and each iteration accesses 
$A$ via a single matrix-vector product.

Let $\lanczos(A, b)$ be the result of $k$ iterations of the Lanczos 
method for approximating $e^A b$. We provide a precise description of the 
method in Appendix~\ref{sec:app-expvec}. Let
\begin{equation}\label{eq:lanczos-proj-def}
\xApp[t;k] = \MapApp[u_t;k]\Bigg(\eta\sum_{i=1}^{t-1}\g[i]\Bigg),
~~\mbox{where}~~
\MapApp(\xM) = 
\frac{vv^T}{v^T v}
~~\mbox{for}~~
v=\lanczos(\xM/2, u)
\end{equation}
denote the \emph{approximate} randomized mirror projection. 
Using the Lanczos method to compute full eigen-decompositions 
has well-documented 
numerical 
stability issues \citep{MeurantGe06}. In contrast, the 
approximation~\eqref{eq:lanczos-proj-def} appears to be numerically 
stable. To provide a theoretical basis for this observation, we exhibit error 
bounds under finite floating point precision, leveraging the 
results 
of~\cite{MuscoMuSi18}, 
 which in turn build on \citet{%
	DruskinKn91,DurskinKn95}. To account for computational cost, we 
	denote by $\matvec(\xM)$  the cost of multiplying matrix $\xM$ by any 
	vector.

\begin{proposition}
	\label{prop:expvec-main}	
	Let $\epsilon, \delta \in (0,1)$ and $\xM \in \domainDual$, and set 
	$M \defeq \max\{\opnorm{A} , \log(\frac{n}{\epsilon\delta}), 1\}$. Let 
	$u$ be uniformly 
	distributed on the unit sphere in $\R^n$ and independent of $\xM$.  
	If the number of 
	Lanczos iterations $k$ satisfies $k \ge \Theta(1) \sqrt{M 
	\log(\frac{nM}{\epsilon \delta})}$ then the 
	approximation~\eqref{eq:lanczos-proj-def} satisfies
	\begin{equation*}
	\lones{\Map(\xM)-\MapApp(\xM)} \le \epsilon
	~\mbox{with probability}~
	\ge 1-\delta
	~\mbox{over}~u\sim \uniform(\sphere^{n-1})
	\end{equation*}
	when implemented using floating point operations with 
	$B = \Theta(1)\log\frac{nM}{\epsilon\delta}$ bits of precision. The 
	time to 
	compute $\MapApp(\xM)$ is 
	$O(\matvec(\xM)k + k^2 B)$.
\end{proposition}
We prove Proposition~\ref{prop:expvec-main} in \Cref{sec:app-expvec} and 
describe 
here the main ingredients in the proof. First, we show by  
calculation that 
\arxiv{\begin{equation*}
\lones{\Map(\xM)-\MapApp(\xM)} \le 
\sqrt{8}\frac{\ltwo{e^{\xM/2}u-\lanczos(\xM/2,u)}}{\ltwo{e^{\xM/2}u}}.
\end{equation*}}
\notarxiv{
$
\lones{\Map(\xM)-\MapApp(\xM)} \le 
\sqrt{8}\frac{\ltwo{e^{\xM/2}u-\lanczos(\xM/2,u)}}{\ltwo{e^{\xM/2}u}}.
$ 
}
Therefore, a multiplicative error guarantee for $\lanczos$ would imply our 
result. Unfortunately, for such a guarantee to hold for \emph{all} vectors 
$u$ we must have $k= \Omega(\linf{\xM})$~\citep[Section 
3.3]{OrecchiaSaVi12}. 
We circumvent that by 
using the
randomness of $u$ to argue that w.h.p. $\ltwos{e^{\xM/2}u} \gtrsim 
\frac{1}{\sqrt{n}}e^{\lambdamax(\xM/2)}\ltwo{u}$. This allows us to use 
existing 
additive error guarantees for $\lanczos$ to obtain our result.

\notarxiv{We connect the approximation to regret in 
Appendix~\ref{sec:app-expvec-cor}. In the setting of 
Corollary~\ref{cor:high-probability-regret}, we show that 
approximating $\x$ via $\xApp[t;k_t]$ with $k_t = O(\ceil{\sqrt{\eta 
t}\;}\log (nT/\delta))$ leaves the regret guarantees essentially unchanged. 
Therefore, 
we may achieve $\epsilon$ average regret with 
$O(\epsilon^{-2.5}\log^{2.5}(\frac{n}{\epsilon\delta}))$
 matrix-vectors product, with probability at least $1-\delta$.

}%
\arxiv{
We connect the approximation to regret in the following corollary 
(see  
Appendix~\ref{sec:app-expvec-cor})
\begin{restatable}{corollary}{restateHighProbApp}
	\label{cor:high-probability-regret-app}
	Let $\g[1],\ldots,\g[T]$ be symmetric gain matrices satisfying 
	 $\linf{\g}\le 1$  
	 for every $t$. There exists a
	numerical constant $k_0<\infty$, such that for  every $T\in\N$ and 
	$\delta \in (0,1)$,  
	$\xApp[t;k_t]$ defined in~\eqref{eq:lanczos-proj-def} with 
	$k_t=\ceil{k_0 (\sqrt{1+\eta t}) \log(\tfrac{nT}{\delta})}$, and 
	$\x$ 
	defined in~\eqref{eq:rand-da} satisfy
	\begin{equation}\label{eq:regret-app-bound}
	\sum_{t=1}^T 
	\inner{\g}{\xApp[t;k_t]}  \ge -1 + \sum_{t=1}^T 
	\inner{\g}{\x} 
	~~\mbox{w.p.}\ge 1-\delta/2.
	\end{equation}
	Let $\eps\in(0,1]$,  
	 $T = \frac{16\log(4en/\delta)}{\eps^2}$ and 
	 $\eta = \sqrt{\frac{2\log(4en)}{3T}}$. If 
	 Assumption~\ref{ass:bandit-adversary} holds with respect to the 
	 actions 
	 $\xApp[t;k_t]$, then with probability at least 
	 $1-\delta$, 
	 $\frac{1}{T}\lambdamax\left( \sum_{i=1}^T \g \right) - 
	 \frac{1}{T}\sum_{t=1}^T \inner{\g}{\xApp[t;k_t]} \le \eps$. Computing 
	 the actions 
	 $\xApp[1;k_1], \ldots, 
	 \xApp[T;k_T]$ requires 
	 $O(\eps^{-2.5}\log^{2.5}(\frac{n}{\epsilon\delta}))$ matrix-vector 
	 products.
\end{restatable}
}%
Finally, as we discuss in detail in \Cref{sec:app-expvec-improve},  
computing matrix exponential-vector products (and hence $\Map$) 
reduces to solving $\Otil{1}$ linear systems. Since \citet{AllenLi17} propose 
to compute their sketch using a similar reduction, the running time 
guarantees they establish for their sketch are also valid for ours.
\section{Application to semidefinite programming}
\label{sec:general-sdp}

\newcommand{\orig}[1]{\tilde{#1}}
\newcommand{\Astar}{\mathcal{A}^\star}
\newcommand{\widthparam}{\omega}

Here we describe how to use our rank-1 sketch to solve semidefinite 
programs (SDPs). 
The standard SDP formulation is, given $\orig{C},\orig{A_1}, \ldots, 
\orig{A}_{\orig{m}} \in 
\domainDual[\orig{n}]$ and $\orig{b}\in\R^{\orig{m}}$,
\begin{equation*}
\label{eq:sdpformulation}
\minimize_{Z\succeq 0} \inners{\orig{C}}{Z}~\textrm{subject to } 
\inners{\orig{A_i}}{Z} 
= \orig{b}_i
~~\forall i\in[\orig{m}].
\end{equation*}
A binary search over the optimum value reduces this problem to a 
sequence of feasibility problems. When the constraints imply 
 $\tr Z \le r$ for some $r<\infty$, 
every intermediate feasibility problem is equivalent to deciding whether 
there exists $\x[]$ in the spectrahedron $\domain$ s.t. $\inner{A_i}{X}\le 
0$ for all $i\in[m]$, with $n,m$ and $A_i\in\domainDual$ constructed 
from $\orig{n},\orig{m},\orig{A}_i, \orig{b}, \orig{C}$ and $r$. This decision 
problem is in turn 
equivalent~\citep[cf.][]{GarberHa16} to determining the sign of
\begin{equation}
\label{eq:spobjective}
\saddlevalue=\min_{y \in \simplex} \max_{X \in \domain} \inner{\Astar 
y}{X},
~\mbox{where}~
\Astar y \defeq \sum_{i \in [m]} [y]_i A_i.
\end{equation}
and $\simplex$ is the simplex in $\R^m$.
\arxiv{
For every $y\in\simplex$ and $\x[]\in\domain$, we have that
\begin{equation*}
\min_{y' \in \simplex} \inners{\Astar y'}{X}
\le
\saddlevalue
\le \max_{X' \in \domain}  
\inners{\Astar y}{X'}.
\end{equation*}
}%
\notarxiv{
We have that 
$
\min_{y' \in \simplex} \inners{\Astar y'}{X}
\le
\saddlevalue
\le \max_{X' \in \domain}  
\inners{\Astar y}{X'}
$
 for every $y\in\simplex$ and $\x[]\in\domain$.
}%
Therefore, to determine $\saddlevalue$ to additive error $\epsilon$, it 
suffices to find $y,X$ with $\dualitygap(X, y)\le\epsilon$, where
\begin{equation}
\label{eq:duality-gap-def}
\dualitygap(X, y)
\defeq
\max_{X' \in \domain}  \inners{\Astar y}{X'}
- \min_{y' \in \simplex} \inners{\Astar y'}{X}
=
\lambdamax\left( \Astar y\right) - \min_{i\in [m]}\inner{A_i}{X}.
\end{equation}

A basic approach to solving convex-concave games such 
as~\eqref{eq:spobjective} is to apply online learning for $X$ and $y$  
simultaneously, 
where at each round the gains/costs to the max/min player are determined by the 
actions of the opposite player in the previous round. Importantly, such dynamics 
satisfy Assumption~\ref{ass:bandit-adversary}, and we use our rank-1 sketch as 
the online learning strategy of the (matrix) max player, and standard 
multiplicative weights for the (vector) min player. Algorithm~\ref{alg:sdp} 
describes the resulting scheme. 
\arxiv{The algorithm entertains a convergence guarantee that depends on 
the \emph{width parameter}
\begin{equation*}
\widthparam \defeq \max_{i\in[m]} \linf{A_i}
\end{equation*}
and has the following form.
}
\notarxiv{
To factor problem scaling into the analysis, we define the \emph{width 
parameter}
\begin{equation*}
\widthparam \defeq \max_{i\in[m]} \linf{A_i}.
\end{equation*}
In \Cref{sec:app-general-sdp-proof} we use the standard no-regret 
argument combined with 
Corollary~\ref{cor:expected-regret} to prove the following converges 
guarantee for \Cref{alg:sdp} (a high-probability version follows  
via Corollary~\ref{cor:high-probability-regret}).
}

\begin{algorithm2e}
	\caption{Primal-dual SDP feasibility}
	\label{alg:sdp}
		Let $\g[0] := 0$ and $c_0 := 0$
		
	\For{$t = 1,\ldots,T$}
	{
		\vspace{3pt}
		
		Sample vector $u_t$ uniformly at random from the unit sphere
		
		\vspace{3pt}
		
		Play matrix $X_t \defeq \Map[u_t]\big(\sum_{i = 1}^{t - 1} \eta 
		\g[i] 
		\big)$
		
		\vspace{2pt}
		
		Play vector $y_t \defeq \gradsoftmax\big(-\eta 
		\sum_{i=1}^{t-1}c_{i}\big)=\frac{
			y_{t-1}\circ e^{-\eta c_{t-1}}}{\ones^T  (y_{t-1}\circ e^{-\eta 
				c_{t-1}})}$.

		\vspace{3pt}
		
		Form gain matrix $G_t = \Astar y_t = \sum_{i\in[m]}[y_t]_i A_i$ 
		
		\vspace{5pt}
		
		Form cost vector
		$[c_{t}]_i \defeq \inner{X_{t}}{A_i}$, $i 
		\in 
		[m]$
	}
\end{algorithm2e}

\begin{restatable}{theorem}{restatepdregret}
\label{thm:gensdp-algorithm-guarantees}
Let $\{X_t, y_t\}_{t = 1}^T$ be the actions produced by Algorithm 1 and, define $\Xavg = 
\frac{1}{T} \sum_{t = 1}^T X_t$, $\yavg= \frac{1}{T} \sum_{t = 1}^T 
y_t$. Then
\begin{equation*}
\E \left[ \dualitygap\left(\Xavg, \yavg\right) \right] \le 
\frac{\log(4mn)}{\eta T} + 
2 \eta \widthparam^2 .
\end{equation*}
\end{restatable}
\notarxiv{
For $\eta = \frac{\log(4mn)}{\sqrt{2 \widthparam^2 T }}$ and $T = \frac{8 
\log (4mn) \widthparam^2}{ \epsilon^2}$, 
Theorem~\ref{thm:gensdp-algorithm-guarantees} guarantees 
$\E\left[\dualitygap(\Xavg, \yavg)\right] \leq \epsilon$. Let $\matvec(M)$ 
denote the time required to multiply $M$ by 
any vector, let $\timeMatvecA \defeq \sum_{i\in [m]} \matvec(A_i)$, and 
let  $\timeMatvecAstar \defeq \max_{\alpha\in\R^m} 
\{\matvec(\Astar \alpha)\} \le \min\{\timeMatvecA, n^2\}$. At each 
iteration, computing $c_t$ and $y_t$ takes $O(\timeMatvecA)$ time while 
computing $\x$ takes 
$\Otil{(\widthparam/\epsilon)^{0.5} \timeMatvecAstar}$ time by 
\Cref{sec:expvec-compute}, and $\g$ need not be formed explicitly. The 
total computational cost is therefore at most
\begin{equation*}
\Otil{\left[  (\widthparam/\eps)^{0.5}\,\timeMatvecAstar + \timeMatvecA 
	\right]T} 
= 
\Otil{ (\widthparam/\eps)^{2.5}\,\timeMatvecAstar +  
	(\widthparam/\eps)^2\,\timeMatvecA}.
\end{equation*}
In \Cref{sec:app-general-sdp-comp,sec:app-general-sdp-runtime} we 
derive this bound in more detail and compare it with the 
literature.
}
\arxiv{
\begin{proof} Recalling the definition~\eqref{eq:duality-gap-def} of the 
duality gap, and that $G_t = \Astar y_{t}$ and $[c_t]_i = 
	\inner{A_i}{X_t}$, we have
	\begin{equation*}
	\dualitygap(\Xavg, \yavg) = \frac{1}{T}\lambdamax\bigg(\sum_{t=1}^T 
	G_t\bigg) 
	- \frac{1}{T} \min_{i\in[m]} \bigg\{ \sum_{t=1}^T [c_t]_i \bigg\}.
	\end{equation*}
	
	Note that $y_t=\gradsoftmax(-\eta \sum_{i=1}^{t-1} c_i)$ is a function 
	of 
	$X_1, \ldots, X_{t-1}$. Therefore, $G_t = \Astar y_t$ satisfies 
	Assumption~\ref{ass:bandit-adversary} and we may use 
	Corollary~\ref{cor:expected-regret} to write
	\begin{equation*}
	\E\bigg[\lambdamax\bigg(\sum_{t=1}^T \g \bigg) - \sum_{t=1}^T 
	\inner{\g}{\x}\bigg]
	\le 
	\frac{\log(4n)}{\eta} + 
	\frac{3\eta}{2} \cdot 
	\sum_{t=1}^T 
	\E\big[
	\linf{\g}^2 \big] \le \frac{\log(4n)}{\eta} + \frac{3\eta \widthparam^2 
	T}{2},
	\end{equation*}
	where in the second inequality we used $\widthparam=\max_{i\in[m]} 
	\linf{A_i}$  
	and 
	$y\in\simplex$ to bound 
	$\linf{\g} = \linf{\Astar y_t} \le \widthparam\cdot \ones^T y_t = 
	\widthparam$. Similarly, we 
	use the standard multiplicative weights regret 
	bound~\citep[cf.][Theorem 2.21]{Shalev12}
	 to write
	\begin{equation*}
	\sum_{t=1}^T c_t^T y_t - \min_{i\in[m]} \bigg\{ \sum_{t=1}^T [c_t]_i 
	\bigg\}
	\le \frac{\log(m)}{\eta} + 
	\frac{\eta}{2} \cdot 
	\sum_{t=1}^T 
	\linf{c_t}^2\le \frac{\log(m)}{\eta} + \frac{\eta \widthparam^2 T}{2},
	\end{equation*}
	where the second inequality again follows from $\left|[c_t]_i\right| = 
	\left| 
	\inner{A_i}{X_t}\right| \le \linf{A_i} \le \widthparam$ since $X_t \in 
	\domain$. 
	
	Finally, 
	\begin{equation*}
	c_t^T y_t = \sum_{i=1}^m [y_t]_i \inner{A_i}{X_t} = \inner{\Astar 
	y_t}{X_t} 
	= 
	\inner{G_t}{X_t}.
	\end{equation*}
	Hence, summing the two regret bounds and dividing by $T$ gives the 
	result.
\end{proof}

For $\eta = \log(4mn)/\sqrt{2 \widthparam^2 T }$ and $T = 8 \log (4mn) 
\widthparam^2 / 
\epsilon^2$, Theorem~\ref{thm:gensdp-algorithm-guarantees} guarantees 
{\linebreak}
$\E\left[\dualitygap(\Xavg, \yavg)\right] \leq \epsilon$. 
A high-probability version of this guarantee follows readily 
via~\Cref{cor:high-probability-regret}. %

Let us now discuss the computational cost of Algorithm~\ref{alg:sdp}.
Let $\matvec(M)$ denote the time required to multiply the matrix $M$ by 
any vector, and let $\timeMatvecA \defeq \sum_{i\in [m]} \matvec(A_i)$.
Except for the computation of $\x$, every step in the for loop 
in Algorithm~\ref{alg:sdp} takes $O(\timeMatvecA)$ work 
to execute (we may assume $\timeMatvecA \ge \max\{n,m\}$ without loss 
of generality). Let $Y_t = \eta \sum_{i=1}^{t-1} \g[i]= \Astar 
(\sum_{i=1}^{t-1} \eta y_i )$, and note that, with the 
values of $\eta$ and $T$ above, $\linf{Y_t} \le \eta T \widthparam = 
\Otil{{\widthparam/\eps}}$ for every $t\le T$.  
Per~\Cref{sec:expvec-compute}, the 
computation of $\x$ costs  
$\Otil{\linf{Y_t}^{0.5}\,\matvec(Y_t)} =  
\Otil{(\widthparam/\eps)^{0.5}\,\matvec(Y_t)}$. Writing  
$\timeMatvecAstar \defeq \max_{\alpha\in\R^m} 
\{\matvec(\Astar \alpha)\} \le \min\{\timeMatvecA, n^2\}$, 
the total computational cost of our 
algorithm is
\begin{equation*}
\Otil{\left[  (\widthparam/\eps)^{0.5}\,\timeMatvecAstar + \timeMatvecA 
\right]T} 
= 
\Otil{ (\widthparam/\eps)^{2.5}\,\timeMatvecAstar +  
(\widthparam/\eps)^2\,\timeMatvecA}.
\end{equation*}
In 
many settings of interest---namely when the $A_i$s have mostly 
non-overlapping sparsity patterns and yet the $Y_t$s are sparse---we have 
$\timeMatvecAstar\approx 
\timeMatvecA$, so that the computational cost is dominated by the first 
term.

\subsection{Comparison with other 
algorithms}\label{sec:app-general-sdp-comp}
Let $\nnz(M)$ denote the number of nonzero entries of matrix $M$, and 
let 
$\timeA \defeq \sum_{i\in [m]} \nnz(A_i) \ge \timeMatvecA$.
If in Algorithm~\ref{alg:sdp} we replace the randomized projection $\Map$ 
with the matrix multiplicative weights projection $\MapMMW$, the regret 
bound of Theorem~\ref{thm:gensdp-algorithm-guarantees} still holds, but 
the overall 
computational cost becomes $\Otil{(\widthparam/\eps)^2\,(n^3+\timeA)}$ 
due to full 
matrix 
exponentiation. \citet{Nemirovski04} accelerates this scheme using 
extra-gradient steps, guaranteeing duality gap below $\eps$ in 
$\Otil{\widthparam/\eps}$ iterations, with each iteration involving two full 
matrix 
exponential computations. The overall computational cost of such scheme 
is consequently $\Otil{(\widthparam/\eps)\,(n^3+\timeA)}$. 
\citet{Nesterov07b} 
attains the 
same rate by using accelerated gradient descent on a smoothed version of 
the dual problem. Our scheme improves on this rate for sufficiently sparse 
problems, 
with $n^3 / \timeA \gg  (\widthparam/\eps)^{-1.5}$.

\cite{d'Aspremont11} applies a 
subgradient 
method to the dual problem, approximating the subgradients using the 
Lanczos method to compute a leading eigenvector of $\Astar y$. The 
method solves the dual problem to accuracy $\epsilon$ with total work 
$\Otil{ (\widthparam/\eps)^{2.5}\,\timeMatvecAstar +  
(\widthparam/\eps)^2\,\timeMatvecA}$, essentially the 
same as us. However, it is not clear how to efficiently recover a primal 
solution from this method. Moreover, the surrogate duality 
gap~\cite{d'Aspremont11} proposes will not always be 0 at the global 
optimum, whereas with our approach the true duality gap is readily 
computable.

\citet{BaesBuNe13} replace the full matrix exponentiation in the accelerated 
scheme of~\citet{Nemirovski04} with a rank-$k$ sketch of the 
form~\eqref{eq:rank-k-sketch-def}, where 
$k=\Otil{\widthparam/\epsilon}$. 
Similarly to~\cite {Nemirovski04}, they require 
$\Otil{\widthparam/\epsilon}$ 
iterations to attain duality gap below $\epsilon$. \citet{BaesBuNe13} 
approximate matrix exponential vector products by truncating a Taylor 
series, costing  
$\Otil{k(\widthparam/\epsilon)\,\timeMatvecAstar}
=\Otil{(\widthparam/\epsilon)^2\,\timeMatvecAstar}$
 work 
per iteration. With the Lanczos method, the cost improves to \linebreak
$\Otil{(\widthparam/\epsilon)^{1.5}\,\timeMatvecAstar}$ work per 
iteration. Every 
step of 
their method also 
computes $\inner{A_i}{X}$ for all $i\in[m]$ and a rank-$k$ matrix 
$X=\sum_{j=1}^k v_j v_j^T$; this costs either $k \cdot \timeMatvecA$ 
work 
(computing $\inner{A_i}{v_j}$ for every $i,j$) or $\timeA + n^2 k$ (when 
forming $X$ explicitly). The former option yields total complexity 
 identical to our 
method. The latter option is preferable 
only when $\timeA \gg n^2 \ge \timeMatvecAstar$, and can result in an 
improvement over the 
running time of our method if 
$\timeMatvecAstar \ll \timeA\,(\widthparam/\eps)^{-1.5} + n^2 
\,(\widthparam/\eps)^{-0.5}$. 
\citet{BaesBuNe13} report that $k=1$ often gave the best result in their 
experiment, 
which is not predicted by their theory. A hypothetical explanation for this 
finding is that, with $k=1$, they are essentially running   
Algorithm~\ref{alg:sdp}.

Finally, \citet{d'Aspremont11} and \citet{GarberHa16} propose sub-sampling
based algorithms for approximate SDP feasibility with runtimes potentially
sublinear in $\timeMatvecAstar$. However, because of their significantly
worse dependence on $\widthparam/\eps$, as well as dependence on Frobenius
norms, we match or improve upon their runtime guarantees in a variety of
settings; see~\citep{GarberHa16} for a detailed comparison.
} %
\section{Discussion}\label{sec:discussion}

\arxiv{
We conclude the paper with a discussion of a number of additional settings 
where our sketch---or some variation thereof---might be beneficial. In the 
first two settings we discuss, the naturally arising online learning problem 
involves adversaries that violate Assumption~\ref{ass:bandit-adversary}, 
demonstrating a limitation of our analysis.
}\notarxiv{
We discuss four additional settings where our sketch might be beneficial. 
In the first two, the associated online learning problem 
involves adversaries that violate Assumption~\ref{ass:bandit-adversary}, 
demonstrating a limitation of our analysis.}

\paragraph{Online convex optimization}
In the online convex optimization problem, at every time step $t$ 
the adversary provides a convex loss $\ell_t$, the players pays a cost 
$\ell_t(\x)$ and wishes to minimize the regret $\sum_{t=1}^T \ell_t(\x) - 
\min_{\x[]} \sum_{t=1}^T \ell_t(\x[])$. The standard reduction to 
the online learning problem is to construct an adversary with gains 
$\g=-\grad \ell_t(\x)$. However, even if the losses $\ell_t$ 
follow~\Cref{ass:bandit-adversary}, the constructed gains $\g$ clearly 
violate it. Therefore, extensions of our results to online convex optimization will 
require additional work and probably depend on finer problem 
structure.

\paragraph{Positive semidefinite programming} \citet{PengTaZh16} 
and~\citet{AllenLeOr16} propose algorithms for solving positive 
(packing/covering) semidefinite programs with width independent running 
time, meaning that the computational cost of solving the problems to 
$\epsilon$ multiplicative error depends only logarithmically on the width 
parameter ($\widthparam$ in Section~\ref{sec:general-sdp}). Both 
algorithms 
rely on matrix exponentiation, which they approximate with a rank 
$\Otil{\eps^{-2}}$ sketch using the Johnson-Lindenstrauss lemma. The 
algorithm of \citet{PengTaZh16} uses matrix multiplicative weights in 
essentially a black-box fashion, so one could hope to replace their 
high-rank sketch with our rank-1 technique. Unfortunately, the gain 
matrices that they construct violate 
Assumption~\ref{ass:bandit-adversary} and so our results do not 
immediately apply. 
A rank-1 sketch for this setting remains an intriguing open problem.

\paragraph{Improved computational efficiency against an oblivious 
	adversary}
An oblivious adversary produces gain matrices $\g[1],\ldots,\g[T]$ 
independent of the actions $\x[1],\ldots,\x[T]$; this is a stronger version 
of \Cref{ass:bandit-adversary}. For such an adversary, if we draw 
$u\sim\uniform(\sphere^{n-1})$ and set $u_1=u_2=\cdots=u_T=u$, the 
average regret guarantee of~\Cref{cor:expected-regret} still applies, 
as~\cite{AllenLi17} explain. In this setting, it may be possible to make the 
computation of $\x$ more efficient by reusing $\x[t-1]$. Such savings 
exist in the stochastic setting (when $\g$ are i.i.d.) via Oja's 
algorithm~\citep{AllenLi17}, and would be interesting to extend to the 
oblivious setting. 

\paragraph{Online $k$ eigenvectors}\cite{NieKoWa13} show that a variant of 
matrix multiplicative weights is also 
capable of learning online the top $k$-dimensional eigenspace, 
with similar regret guarantees. As our rank-1 sketch solves the $k=1$ leading 
eigenvector problem, it is interesting to study whether a rank-$k$ sketch solves 
the $k$ 
leading eigenvectors problem.

\newpage

\appendix

\section{Dual averaging regret bounds}\label{sec:app-da-regret-proof}

\restateAvRegret*
\begin{proof} 
	We start with the well-known Bregman 3-point identity, valid for any 
	$\Phi_0, \Phi_1, \Phi_2\in\domainDual$,
	\begin{equation}
	\label{eq:threepoint}
	\inner{\Phi_2 - \Phi_1}{\MapAv(\Phi_0) - \MapAv(\Phi_1)}
	=
	\breg{\Phi_0}{\Phi_1} -\breg{\Phi_0}{\Phi_2}  + \breg{\Phi_1}{\Phi_2};
	\end{equation}
	the identity follows from the definition~\eqref{eq:bregman-def} of 
	$\bregBlank$ by direct substitution. Fix some $S\in\interior\domain$ 
	and 
	$S\in\domainDual$ such that $S = \MapAv(\Psi)$ (which exists by 
	Proposition~\ref{prop:regret-properties}.\ref{item:invert}).
	Let $\xM_t = \eta \sum_{i=1}^{t-1}\g[i]$ so that $\xAv = 
	\MapAv(\xM_t)$. 
	For a given $t$, we use the 3-point identity with $\Phi_0 = \Psi, 
	\Phi_1 = \xM_{t}$ and $\Phi_2 = \xM_{t+1}$, yielding
	\begin{equation*}
	\eta \inner{\g}{S - \xAv} = 
	\inner{\xM_{t+1}-\xM_{t}}{\MapAv(\Psi)-\MapAv(\xM_t)}
	=
	\breg{\Psi}{\xM_{t}} -\breg{\Psi}{\xM_{t+1}}  + \breg{\xM_{t}}{\xM_{t+1}}.
	\end{equation*}
	Summing these equalities over $t=1,\ldots,T$ and dividing by $\eta$ 
	gives 
	\begin{flalign}
	\inner{\sum_{t=1}^T \g}{S} - \sum_{t=1}^T \inner{\g}{\xAv}
	& = \frac{\breg{\Psi}{\xM_1} -\breg{\Psi}{\xM_{T+1}}}{\eta} 
	+ \frac{1}{\eta} \sum_{t=1}^T \breg{\xM_{t}}{\xM_{t+1}}
	\label{eq:regret-key-identity} \\& \le \frac{\log 4n}{\eta} + 
	\frac{3\eta}{2}
	\sum_{t=1}^T 
	\linf{\g}^2.
	\label{eq:regret-bound-basic-intermediate}
	\end{flalign}
	Above, we used $\breg{\xM_{t}}{\xM_{t+1}} = \breg{\xM_{t}}{\xM_{t} + 
	\eta \g}
	\le \frac{3}{2} \eta^2 \linf{\g}^2$ 
	(Proposition~\ref{prop:regret-properties}.\ref{item:smoothness}) along 
	with $\xM_1 = 0$ and $\breg{\Psi}{0} 
	-\breg{\Psi}{\xM_{T+1}} \le \log 4n$ 
	(Proposition~\ref{prop:regret-properties}.\ref{item:diameter}).

	Since the bound~\eqref{eq:regret-bound-basic-intermediate} is valid for 
	any $S\in\interior\domain$, we may supremize it over $S$. The 
	result~\eqref{eq:regret-bound-basic} follows from noting that 
	$\sup_{S\in\interior\domain}\inners{\sum_{t=1}^T \g}{S} = 
	\lambdamax\big( \sum_{i=1}^T \g \big)$.

	To see the second bound~\eqref{eq:regret-bound-refined}, 
	we return to 
	the identity~\eqref{eq:regret-key-identity} and note that the  
	assumptions  
	$0 \preceq \g \preceq I$ and $\eta \le\frac{1}{6}$ imply 
	$\linf{\eta 
		\g} 
	\le \frac{1}{6}$. Therefore we may use 
	Proposition~\ref{prop:regret-properties}.\ref{item:refined-smoothness}
	 to obtain
	\begin{flalign*}
	\breg{\xM_{t}}{\xM_{t+1}} = \breg{\xM_{t}}{\xM_{t} + \eta \g}
	& \le 3\linf{\eta \g} \inner{\eta \g}{\MapAv(\xM_t)} 
	=
	3\eta^2 \inner{\g}{\xAv}.
	\end{flalign*}
	Substituting back 
	into~\eqref{eq:regret-key-identity}, rearranging and taking the supremum 
	over $S$ as before, we obtain
	\begin{equation}\label{eq:regret-bound-refined-var}
	\lambdamax\left( \sum_{i=1}^T \g \right)
	\le \left(1 + 3\eta\right) 
	\sum_{t=1}^T 
	\inner{\g}{\xAv} 
	+ \frac{\log(4n)}{\eta}.
	\end{equation}
	Dividing through by  $\left(1 + 3\eta\right)$ and noting that 
	$1-x\le \frac{1}{1+x}\le1$ for every $x\ge 0$, we obtain the 
	result~\eqref{eq:regret-bound-refined}, 
	concluding the proof.
\end{proof} %
\section{High probability regret bounds}
\label{sec:app-regret-hp}

\restateHighProb*

\begin{proof}
We start with the first claim~\eqref{eq:regret-bound-basic-high-prob}.
Recall that a random process $D_t$ adapted to a filtration $\Filt[t]$
is $\sigma^2$-sub-Gaussian if
$\E[\exp(\lambda D_t) \mid \Filt[t-1]] \le
\exp(\lambda^2 \sigma^2 / 2)$ for all $\lambda \in \R$.
Then using the boundedness assumption
that
$\<\g[t], \x[t]\> \le \linf{\g[t]} \le 1$,
Hoeffding's lemma on bounded random variables~\citep{Hoeffding63}
implies that
the martingale difference sequence
$\<\g[t], \x[t] - \xAv[t]\>$ is $1$-sub-Gaussian.
Consequently,
the Azuma-Hoeffding inequality~\citep{Azuma67} immediately implies that
\begin{equation*}
  \sum_{t = 1}^T \<\g[t], \x[t]\>
  \ge \sum_{t = 1}^T \<\g[t], \xAv[t]\> - \sqrt{2 T \log \tfrac{1}{\delta}}
  ~~ \mbox{w.p.} ~ \ge 1 - \delta.
\end{equation*}
The bound~\eqref{eq:regret-bound-basic}
in Theorem~\ref{thm:av-regret} thus gives the 
result~\eqref{eq:regret-bound-basic-high-prob}.

For the multiplicative bound~\eqref{eq:regret-bound-refined-high-prob}, we
require a slightly different relative martingale convergence guarantee.

\begin{lemma}[\citet{AllenLi17}, Lemma~G.1]
  \label{lem:positive-chernoff}
  Let $\{D_t\}$ be adapted to the filtration $\{\Filt[t]\}$ and satisfy
  $0\le D_t \le 1$. Then, for any
  $\delta,\mu\in(0,1)$, and any $T\in\N$,
  \begin{equation*}
    \P\left( \sum_{t=1}^T D_t \ge  (1-\mu)\sum_{t=1}^T \E\left[D_t 
      \mid 
      \Filt[t-1] \right] - \frac{\log\tfrac{1}{\delta}}{\mu}\right)
    \ge 1-\delta.
  \end{equation*}
\end{lemma}

Similarly, the assumption $0\preceq\g\preceq I$, along with 
$\x\in\domain$, imply $0\le\inner{\g}{\x}\le1$. Therefore,  the conditions 
of Lemma~\ref{lem:positive-chernoff} 
hold for $D_t = \inner{\g}{\x}$, and we use it with $\mu=\eta\le 1$, 
obtaining
\begin{equation*}
  \sum_{t=1}^T\inner{\g}{\x} \ge (1-\eta)\sum_{t=1}^T\inner{\g}{\xAv} - 
  \frac{\log\tfrac{1}{\delta}}{\eta}
  ~\mbox{w.p.}~\ge 1-\delta.
\end{equation*}
The bound~\eqref{eq:regret-bound-refined} in 
Theorem~\ref{thm:av-regret}
thus yields that with probability at least $1 - \delta$
over the randomness in $\x[t]$ and $\g[t]$,
\begin{equation*}
  \sum_{t=1}^T\inner{\g}{\x} \ge (1-\eta)(1-3\eta)\lambdamax\left( 
  \sum_{t=1}^T \g \right) 
  -\frac{\log(4n/\delta)}{\eta}.
\end{equation*}
Noting that $(1-\eta)(1-3\eta) \ge 1-4\eta$
completes the proof.
\end{proof}

\section{Proofs from 
Section~\ref{sec:proj-geometry}}\label{sec:app-proj-geometry}

\subsection{Proof of Lemma~\ref{lem:mmw-spectral-bound}}
\label{sec:app-mmw-spec-bound-proof}

\restateMmwSpectralBound*
\begin{proof}
	While the result is evident from the development in~\citep{Nesterov07b}, it is 
	not stated 
	there formally. We therefore derive it here using our notation and one key 
	lemma from~\citep{Nesterov07b}. First, note that
	\begin{equation*}
	\inner{\dM}{\grad\mapMMW(\xM)} = \inner{\dM}{\MapMMW(\xM)}
	= \frac{\inner{\dM}{e^{\xM}}}{\tr e^{\xM}},
	\end{equation*}
	where throughout $\grad$ denotes differentiation with respect to 
	$\xM$ and $\dM$ is viewed as fixed. Applying $\grad$ again gives,
	\begin{equation*}
	\hess\mapMMW(\xM)[\dM, \dM]
	= \inner{\dM}{\grad\Bigg(\frac{\inner{\dM}{e^{\xM}}}{\tr 
	e^{\xM}}\Bigg)}
	= \frac{\inner{\dM}{\grad\inner{\dM}{e^\xM}}}{\tr e^{\xM}}
	- \left( \frac{\inner{\dM}{e^\xM}}{\tr e^{\xM}}\right)^2
	\le \frac{\inner{\dM}{\grad\inner{\dM}{e^\xM}}}{\tr e^{\xM}}.
	\end{equation*}
	Note that $\grad\inner{\dM}{e^\xM}\ne  \dM e^\xM $ when 
	$\dM$ and 
	$\xM$ do not commute. However, using the Taylor series for the 
	exponential and the formula $\grad \inner{\dM}{\xM^k} = 
	\sum_{i=0}^{k-1}\xM^i \dM \xM^{k-1-i}$ gives,
	\begin{equation*}
	\grad\inner{\dM}{e^\xM} = \sum_{k=0}^\infty \frac{1}{k!} \grad 
	\inner{\dM}{\xM^k} = 
	\sum_{k=1}^\infty \sum_{i=0}^{k-1}\frac{1}{k!}  \xM^i \dM 
	\xM^{k-1-i}.
	\end{equation*}
	Consequently, we may write
	\begin{equation*}
	\inner{\dM}{\grad\inner{\dM}{e^\xM}}
	= \sum_{k=1}^\infty \sum_{i=0}^{k-1}\frac{1}{k!}  \inner{\dM}{\xM^i 
	\dM \xM^{k-1-i}}
	= \sum_{k=1}^\infty \sum_{i=0}^{k-1}\frac{1}{2(k!)}  
	\inner{\dM}{\xM^i 
	\dM \xM^{k-1-i}+\xM^{k-1-i} \dM \xM^{i}}.
	\end{equation*}
	Lemma 1 in~\citep{Nesterov07b} shows that, when $\xM\succeq 0$,
	\begin{equation*}
	\inner{\dM}{\xM^i 
		\dM \xM^{k-1-i}+\xM^{k-1-i} \dM \xM^{i}} \le 
		2\inner{\dM^2}{\xM^{k-1}}.
	\end{equation*}
	Substituting back, this gives
	\begin{equation*}
	\inner{\dM}{\grad\inner{\dM}{e^\xM}}
	\le
	 \sum_{k=1}^\infty \frac{1}{(k-1)!}  \inner{\dM^2}{\xM^{k-1}}
	= \inner{\dM^2}{e^{\xM}},
	\end{equation*}
	and consequently
	\begin{equation*}
	\hess\mapMMW(\xM)[\dM, \dM] \le 
	\frac{\inner{\dM^2}{e^{\xM}}}{\tr e^\xM} = 
	\inner{\dM^2}{\MapMMW(\xM)}
	= \inner{\dM^2}{\grad\mapMMW(\xM)}
	\end{equation*}
	as required. Finally, note that the assumption $\xM\succeq0$ is without 
	loss of generality, as $\MapMMW(\xM) = \MapMMW(\xM+cI)$ for 
	every $c\in \R$, and therefore $\hess \mapMMW$  is also invariant 
	to scalar shifts.
\end{proof}

\subsection{Proof of Lemma~\ref{lem:main-spec-bound}}
\label{sec:app-main-spec-bound-proof}
\restateMainSpecBound*
\notarxiv{
\subsubsection{Proof overview}
Before going into the lengthy argument, let us briefly survey its key 
components. Our starting point is the formula~\citet{LewisSe01} provide 
for the Hessian of spectral functions. Combined with combined with the 
spectral characterization~\eqref{eq:softmax-av-identity}, the formula gives 
that
\begin{equation*}
\hess \mapAv(\xM)[\dM, \dM] = 
\diag(\tilde\dM)^T \big[ \E_w\hess \softmax(\lambda + \log w) \big]
\diag(\tilde\dM)
+ 
\inner{\E_w A^w(\lambda)}{\tilde{\dM} \circ \tilde{\dM}}.
\end{equation*}
where $\tilde{D}=Q^T D Q$, $\diag(\tilde{D})\in\R^n$ is the vector 
containing the 
diagonal entries of $\tilde{D}$, $A\circ B$ denotes elementwise 
multiplication of $A$ and $B$,  
and $A_{ij}^w(\lambda) \defeq \frac{\grad_i \softmax(\lambda+\log(w)) - 
	\grad_j \softmax(\lambda+\log(w))}{\lambda_i -\lambda_j}\indic{i\ne 
	j}$. 
With the shorthand $\xM_{\{w\}}\defeq  \xM + Q\diag(\log w)Q^T$, we  
use the formula of~\citet{LewisSe01} again to express 
$\hess\mapMMW(\xM_{\{w\}})$ as 
\begin{equation*}
\hess \mapMMW(\xM_{\{w\}})[\dM, \dM] = 
\diag(\tilde\dM)^T \big[\hess \softmax(\lambda + \log w) \big]
\diag(\tilde\dM)
+ 
\inner{ A^\ones(\lambda+\log w)}{\tilde{\dM} \circ \tilde{\dM}},
\end{equation*}
where $A^{\ones}=A^{\tilde{w}}$ evaluated at $\tilde{w} = 
\ones$.
The bulk of the proof is dedicated to establishing the entry-wise bounds 
\begin{equation*}
\E_w A_{ij}^w(\lambda) \le 
\E_w \left[\left( 1+ 
\frac{\tanh\big(\frac{\lambda_i-\lambda_j}{2}\big)\big| \log  
	\frac{w_i}{w_j} \big|}{\lambda_i - \lambda_j}
\right)
A^{\ones}_{ij}(\lambda + \log w)\right]
\le 3\cdot \E_w A^{\ones}_{ij}(\lambda + \log w).
\end{equation*}
The first inequality follows from pointwise analysis of a symmetrized 
version of $A_{ij}^{w}$. The second inequality follows from piecewise 
monotonicity of $A_{ij}^{w}$ as a function of 
$\log\frac{w_i}{w_j}\sim\mathrm{logit}\,\mathrm{Beta}(\half, \half)$, 
combined with tight exponential tail bounds for the latter. Substituting the 
bound on $\E_w A_{ij}^w(\lambda)$ into the expression for $\hess 
\mapAv(\xM)$ and comparing with $\E_w \hess 
\mapMMW(\xM_{\{w\}})$ yields the desired 
result~\eqref{eq:main-spec-bound1}. 
Applying Lemma~\ref{lem:mmw-spectral-bound} and recalling the 
identity~\eqref{eq:proj-av-identity} yields
\begin{equation*}
\E_w \hess 
\mapMMW(\xM + 
Q\,\diag(\log w)Q^T)[\dM, \dM]
\le \inner{\dM^2}{\E_w \MapMMW(Y+Q\diag(\log w)Q^T)}
=  \inner{\dM^2}{\MapAv(Y)},
\end{equation*}
establishing the final bound~\eqref{eq:main-spec-bound2}.

\subsubsection{Full proof}
}

\renewcommand{\AMMW}{A^{\mathrm{mw}}}
\newcommand{\wswap}{w^{i\leftrightarrow j}}

\arxiv{\begin{proof}}
Let $\tilde{\dM}=Q^T\dM Q$, where as before 
$\xM=Q\Lambda Q^T$ is an eigen-decomposition and $\Lambda = 
\diag(\lambda)$. Recall that $\smapMMW:\R^n \to \R$ denotes the 
vector softmax function, $\smapMMW(y)\defeq \log(\sum_{i=1}^n 
e^{y_i}) =
\mapMMW(\diag y)$. Similarly, define 
$\smapAv(y) \defeq\E_w \softmax(y+\log w)$ for 
$w\sim\dirichlet$. By Lemma~\ref{lem:w-characterization}, $\mapAv(\xM) 
= \smapAv(\lambda)$ is a spectral function. \citet[][Theorem 
3.3]{LewisSe01} prove that
\begin{equation}\label{eq:map-av-hess-expression}
\hess \mapAv(\xM)[\dM, \dM] = 
\hess \smapAv(\lambda)[\diag\tilde\dM,\diag\tilde\dM]  + 
\inner{\AAv(\lambda)}{\tilde{\dM} \circ \tilde{\dM}},
\end{equation}
where $\circ$ denotes elementwise multiplication, $\diag(\tilde{\dM})$ is a 
vector comprised of the diagonal of 
$\tilde{\dM}$, and the 
matrix $\AAv$ is 
given by
\begin{equation*}
\AAv_{ij}(\lambda) = \frac{\grad_i \smapAv(\lambda) - \grad_j 
	\smapAv(\lambda)}{\lambda_i - \lambda_j} 
= \E_w 
\underset{\defeq A_{ij}^w(\lambda)}{
	\underbrace{
		\frac{\grad_i \smapMMW(\lambda+\log w) - \grad_j 
			\smapMMW(\lambda+\log w)}{\lambda_i - \lambda_j}
	}
}
\end{equation*}
for $i\ne j$ and $0$ otherwise, whenever $\lambda$ has distinct 
elements. This distinctiveness assumption is without loss of generality, as 
$\mapAv$ is $\mathcal{C}^2$~\citep[Theorem 4.2]{LewisSe01} so we may 
otherwise consider an arbitrarily small perturbation of $\lambda$ and 
appeal to continuity of $\hess\mapAv$.

We now use the spectral function Hessian formula to write down 
$\hess 
\mapMMW(\xM_{\{w\}})[\dM, \dM]$ where $\xM_{\{w\}}\defeq \xM + 
Q
\diag(\log w) Q^T$ (noting that $\xM$ and  $\xM_{\{w\}}$ have the same 
eigenvectors),
\begin{equation}\label{eq:mmw-hess-expression}
\hess \mapMMW(\xM_{\{w\}})[\dM, \dM] = 
\hess \smapMMW(\lambda+\log w)[\diag\tilde\dM,\diag\tilde\dM]  + 
\inner{\AMMW(\lambda+\log w)}{\tilde{\dM} \circ \tilde{\dM}},
\end{equation}
where
\begin{equation*}
\AMMW_{ij}(\lambda) \defeq  \frac{\grad_i \smapMMW(\lambda) - 
	\grad_j \smapMMW(\lambda)}{\lambda_i - \lambda_j} = 
	A^{\ones}_{ij}(\lambda)
\end{equation*}
for $i\ne j$ and $0$ otherwise. 
Taking the expectation over $w$ in~\eqref{eq:mmw-hess-expression} and 
recalling the definition  
$\smapAv(\lambda) = \E_w \smapMMW(\lambda+\log w)$ gives
\begin{equation}\label{eq:mmw-exp-hess-expression}
\E_w\hess \mapMMW(\xM_{\{w\}})[\dM, \dM] = 
\hess \smapAv(\lambda)[\diag\tilde\dM,\diag\tilde\dM]  + 
\inner{\E_w\AMMW(\lambda+\log w)}{\tilde{\dM} \circ \tilde{\dM}}.
\end{equation}

Comparing  Eq.~\eqref{eq:mmw-exp-hess-expression}  
to~\eqref{eq:map-av-hess-expression} and the desired 
bound~\eqref{eq:main-spec-bound1}, we 
see that it remains to upper bound $\AAv(\lambda)=\E_w A^w(\lambda)$ 
in terms of $\E_w \AMMW(\lambda + \log w)$. 
Fix indices $i,j\in[n]$ such 
that $i\ne j$, and let
\begin{equation*}
\delta \defeq \frac{\lambda_i - \lambda_j}{2}
~~\mbox{and}~~
\rho \defeq \half \log \frac{w_i}{w_j}.
\end{equation*}
Since $\AAv$ and $\AMMW$ are both symmetric matrices, we may assume 
 that $\lambda_i > \lambda_j$ and so $\delta 
>0$ (recall we assumed $\lambda_i \ne \lambda_j$ without loss of 
generality). Let $\wswap$ denote a vector identical to $w$ except 
coordinates 
$i$ and $j$ are swapped. With this notation, 
Lemma~\ref{lem:main-spec-bound-aux}, which we 
prove in 
Section~\ref{sec:app-spec-bound-aux-proof}, yields the bound 
\begin{equation*}
 A_{ij}^w(\lambda) + A_{ij}^{\wswap}(\lambda) \le 
\left( 1+ \frac{|\rho|\tanh(\delta)}{\delta} \right)
\left[\AMMW_{ij}(\lambda + \log w)+
\AMMW_{ij}(\lambda + \log \wswap)\right].
\end{equation*}
Taking the expectation over $w$ and using the fact that $\dirichlet$ is 
invariant to 
permutations, we have
\begin{equation}\label{eq:A-expectation-bound}
\AAv_{ij}(\lambda) \le 
\E_w \left[\left( 1+ \frac{|\rho|\tanh(\delta)}{\delta}
\right) 
\AMMW_{ij}(\lambda + \log w)\right].
\end{equation}

\newcommand{\wnotij}{w_{\setminus ij} }

We now focus on the term $\E_w  \frac{|\rho|\tanh(\delta)}{\delta}
\AMMW_{ij}(\lambda + \log w)$. We have
\begin{flalign}
\E_w  \frac{|\rho|\tanh(\delta)}{\delta}
\AMMW_{ij}(\lambda + \log w)
 =
\E_w  \frac{|\rho|\tanh(\delta)}{\delta}
\AMMW_{ij}(\lambda + \log w)\left[
\indic{|\rho|\le\delta} + \indic{|\rho|>\delta} 
\right]
& \nonumber \\ 
\le 
(\tanh \delta ) \E_w  \AMMW_{ij}(\lambda + \log w)\indic{|\rho|\le\delta}
+ \frac{\tanh \delta}{\delta} \E_w |\rho|
\AMMW_{ij}(\lambda + \log w)\indic{|\rho|>\delta},  &
\label{eq:A-E-rho-tanh-bound}
\end{flalign}
where the final transition uses $|\rho|\indic{|\rho|\le\delta} \le 
\delta\indic{|\rho|\le\delta}$ 
and 
$\AMMW_{ij}(\zeta)\ge 0$ for 
every $\zeta\in\R^n$. The latter is a consequence of the convexity of 
$\smapMMW$ 
and is also evident from Eq.~\eqref{eq:Aij-expression} in  
Section~\ref{sec:app-spec-bound-aux-proof}. 

Since $w\sim \dirichlet$, 
$\rho = \half \log\frac{w_i}{w_j}$ is independent of $\wnotij 
\defeq \{w_k\}_{k\ne i,j}$. Moreover, $w_i,w_j$ are completely determined 
by $\rho$ and $\wnotij$ (see explicit expression in 
Section~\ref{sec:app-spec-bound-mono}). Therefore, conditional on 
$\wnotij$,  $\AMMW_{ij}(\lambda + \log w)$ is a function of $\rho$. 
In 
Lemma~\ref{lem:Ammw-mono} we prove that for every $\lambda$ and 
$\wnotij$, this function is  
decreasing in $\rho$ for $\rho > \delta$. Hence, conditionally on 
$\wnotij$ 
and the event $\rho > \delta$,  
the 
random variables $|\rho|$ and $\AMMW_{ij}(\lambda + 
\log w)$ are \emph{negatively correlated}: the expectation of their product 
at most the product of their expectations. Let $\E_\rho$ 
denote expectation conditional on $\wnotij$. 
Lemma~\ref{lem:mono-neg-corr}, with $f(\rho) = |\rho|$, $g(\rho) = 
\AMMW_{ij}(\lambda + \log w)$, and 
$\Sset=\{\rho\mid \rho > \delta\}$ gives that
\begin{flalign}
\E_\rho |\rho| \AMMW_{ij}(\lambda + \log w)\indic{\rho>\delta}
\le 
\left( \E_\rho  \left[ \, |\rho| \mid \rho > \delta \right]\right)
\left( \E_\rho  \AMMW_{ij}(\lambda + \log w) \indic{\rho > \delta}  \right).
\label{eq:A-mono-bound1}
\end{flalign}
Similarly, Lemma~\ref{lem:Ammw-mono} also gives that (conditional on 
$\wnotij$) $\AMMW_{ij}(\lambda + \log w)$ is increasing in $\rho$ for 
$\rho<-\delta$, and therefore, by Lemma~\ref{lem:mono-neg-corr},
\begin{flalign}
\label{eq:A-mono-bound2}
\E_\rho |\rho| \AMMW_{ij}(\lambda + \log w)\indic{\rho<-\delta}
\le 
\left( \E_\rho   \left[ \, |\rho| \mid \rho < -\delta\right]\right)
\left( \E_\rho  \AMMW_{ij}(\lambda + \log w)\indic{\rho<-\delta}  \right).
\end{flalign}

 Let $z\sim\mathrm{Beta}(\half,\half)$. The random variable $\rho=\half 
 \log\frac{w_i}{w_j}$ is symmetric and distributed as $\half 
 \log(\frac{1-z}{z})$. 
Therefore
\begin{equation*}
\E_\rho  \left[ \, |\rho| \mid \rho < - \delta \right] = 
\E_\rho  \left[ \, |\rho| \mid \rho > \delta \right]  =
\half \E \left[ \log\tfrac{1-z}{z} \mid \log\tfrac{1-z}{z} > 2\delta \right]
\stackrel{(\star)}{\le} \delta + \sqrt{1+e^{-2\delta}},
\end{equation*}
where we prove the inequality $(\star)$ in 
Lemma~\ref{lem:beta-logit-cond-exp}. Substituting this bound into 
inequalities~\eqref{eq:A-mono-bound1} and~\eqref{eq:A-mono-bound2} 
and summing them, we obtain
\begin{equation*}
\E_\rho |\rho| \AMMW_{ij}(\lambda + \log w)\indic{|\rho|>\delta}
\le \left(\delta + \sqrt{1+e^{-2\delta}}\right) \E_\rho \AMMW_{ij}(\lambda 
+ \log 
w)\indic{|\rho|>\delta}.
\end{equation*}
Taking expectation over $\wnotij$ and substituting back 
into~\eqref{eq:A-E-rho-tanh-bound} therefore gives,
\begin{equation*}
\E_w  \frac{|\rho|\tanh(\delta)}{\delta}
\AMMW_{ij}(\lambda + \log w)
\le 
\left( \tanh(\delta) + \sqrt{1+e^{-2\delta}} \cdot 
\frac{\tanh(\delta)}{\delta}\right)
\E_w  \AMMW_{ij}(\lambda + \log w),
\end{equation*}
where we used again $\AMMW_{ij}(\cdot) \ge 0$ in order to increase the 
multiplier of $\E_w  \AMMW_{ij}(\lambda + \log 
w)\indic{|\rho|\le\delta}$. 
Computation shows that $\tanh(\delta) + \sqrt{1+e^{-2\delta}} \cdot 
\frac{\tanh(\delta)}{\delta} \le 1.58 \le 2$ for every $\delta \ge 0$. 
Therefore, by the bound~\eqref{eq:A-expectation-bound} we have
\begin{equation}\label{eq:A-expectation-bound-mu}
\AAv_{ij}\left(\lambda\right)\le
3\cdot \E_{w}A_{ij}\left(\lambda+\log w\right).
\end{equation}

Returning to~\eqref{eq:map-av-hess-expression}, we  
write
\begin{flalign*}
\hess \mapAv(\xM)[\dM, \dM]  \le
\hess \smapAv(\lambda)[\diag\tilde\dM,\diag\tilde\dM]  + 
3
\inner{\E_w \AMMW(\lambda+\log w)}{\tilde{\dM} \circ \tilde{\dM}}
&\\  %
\le 
3\left[ 
\hess \smapAv(\lambda)[\diag\tilde\dM,\diag\tilde\dM]  + 
\inner{\E_w \AMMW(\lambda+\log w)}{\tilde{\dM} \circ \tilde{\dM}}
\right].&
\end{flalign*}
In the first inequality above, we substituted the 
bound~\eqref{eq:A-expectation-bound-mu}, using the fact that all the 
entries of $\tilde{\dM}\circ\tilde{\dM}$ are nonnegative. In the second 
 inequality, we used that fact that 
$\hess\smapAv(\lambda)[\diag\tilde\dM,\diag\tilde\dM]\ge0$ since 
$\smapAv$ is convex. Recalling the 
expression~\eqref{eq:mmw-exp-hess-expression} 
gives~\eqref{eq:main-spec-bound1}. The final 
bound~\eqref{eq:main-spec-bound2} follows from applying 
Lemma~\ref{lem:mmw-spectral-bound} to the right side 
of~\eqref{eq:main-spec-bound1} and using the 
identity~\eqref{eq:proj-av-identity}.
\arxiv{\end{proof}}

\subsubsection{A pointwise bound for 
Lemma~\ref{lem:main-spec-bound}}\label{sec:app-spec-bound-aux-proof}

In this section we prove an elementary inequality that plays a central 
role in 
the proof 
of Lemma~\ref{lem:main-spec-bound}. Let $i,j\in[n]$  be such that $i\ne 
j$. For  
$\lambda\in\R^n$, we define 
\begin{equation}\label{eq:Nij-def}
N_{ij}(\lambda) \defeq \grad_i \smapMMW(\lambda) - \grad_j 
\smapMMW(\lambda)
= \frac{e^{\lambda_i} - e^{\lambda_j}}{\sum_{k=1}^ne^{\lambda_k}}
= \frac{\sinh\left(\frac{\lambda_i-\lambda_j}{2}\right)
}{
	\cosh\left(\frac{\lambda_i-\lambda_j}{2}\right) + \half \sum_{k\ne 
	i,j} 	
	e^{\lambda_k - \frac{\lambda_i+\lambda_j}{2}}}
\end{equation}
and
\begin{equation}\label{eq:Aij-expression}
\AMMW_{ij}(\lambda)  = \frac{N_{ij}(\lambda)}{\lambda_i -\lambda_j}
~~\mbox{and}~~
A_{ij}^w(\lambda) = \frac{N_{ij}(\lambda+\log w)}{\lambda_i -\lambda_j}.
\end{equation}
Additionally, for any vector $w\in\R^n$, 
let $\wswap$ denote a vector identical to $w$ except coordinates $i$ and 
$j$ are swapped. With this notation in hand, we 
state and prove our bound.

\begin{lemma}\label{lem:main-spec-bound-aux}
	Let $\lambda\in \R^n$, $w\in\R^n_+$ and $i,j\in[n]$,  $i\ne j$. 
	Set 
	$\delta=\frac{\lambda_i - \lambda_j}{2}$ and
	$\rho = \half \log\frac{w_i}{w_j}$. Then, 
	 \begin{equation*}
	 A_{ij}^w(\lambda) + A_{ij}^{\wswap}(\lambda) \le 
	 \left( 1+ \frac{|\rho|\tanh(\delta)}{\delta} \right) 
	 \left[\AMMW_{ij}(\lambda + \log w)+
	 \AMMW_{ij}(\lambda + \log \wswap)\right].
	 \end{equation*}
\end{lemma}

\newcommand{\const}{r}

\begin{proof}
	Define
	\begin{equation*}
	\const=\half\sum_{k\notin \{i,j\}} e^{\lambda_k + \log w_k - 
	\frac{\lambda_i + \log w_i + \lambda_j + \log w_j}{2}}\ge0.
\end{equation*}
Observe that if we swap $w_i$ and $w_j$, $\delta$ and $\const$ remain 
unchanged and the sign of $\rho$ reverses.  For $x\in\R$, let $f(x)\defeq 
\frac{\sinh(x)}{\cosh(x) + \const}$.
Using~\eqref{eq:Nij-def}, we may write 
\begin{equation*}
q_1 \defeq 2 A_{ij}^w(\lambda) + 2 A_{ij}^{\wswap}(\lambda) = 
\frac{f(\delta+\rho)}{\delta} + \frac{f(\delta-\rho)}{\delta} 
\end{equation*}
and 
\begin{flalign*}
q_2 \defeq  
2 \AMMW_{ij}(\lambda + \log w)+ 2\AMMW_{ij}(\lambda + \log 
\wswap)
= \frac{f(\delta+\rho)}{\delta+\rho} + \frac{f(\delta-\rho)}{\delta-\rho}. 
\end{flalign*}
With these definitions, our goal is to prove that $\frac{q_1 - q_2}{q_2} \le 
\frac{|\rho|\tanh(\delta)}{\delta}$. 
Since $f(x)$ is an odd function of $x$, the terms $q_1$ and $q_2$ are 
invariant to sign flips in either $\delta$ or $\rho$. Therefore, we may 
assume both 
\begin{equation*}
\text{$\delta\ge0$ and $\rho\ge0$}
\end{equation*}
 without loss of generality.  
 
 Substituting back the expressions for $q_1, 
 q_2$ and using that $|\rho|=\rho$ by assumption yields
 \begin{equation}\label{eq:pointwise-desired-equiv}
 \frac{q_1 - q_2}{q_2} = 
 \frac{\rho}{\delta}\cdot \frac{ \frac{f(\delta+\rho)}{\delta+\rho} -  
 \frac{f(\delta-\rho)}{\delta-\rho} }{
 	\frac{f(\delta+\rho)}{\delta+\rho} + \frac{f(\delta-\rho)}{\delta-\rho}}
 =
  \frac{\rho}{\delta}\cdot \frac{ g(\delta+\rho)-g(\delta-\rho)}{ 
  g(\delta+\rho)+g(\delta-\rho)},
 \end{equation}
 where
 \begin{equation*}
 g(x) \defeq \frac{f(x)}{x} = \frac{\tanh(x)}{x}\cdot\frac{\cosh(x)}{\cosh(x) 
 	+ \const}.
 \end{equation*}
Note that 
 $\frac{\tanh(x)}{x}$ is decreasing in $|x|$. Since 
 $|\delta-\rho|\le|\delta+\rho|$ 
 by the assumption $\rho,\delta\ge0$, we have
 \begin{equation*}
 g(\delta -\rho) \ge
  \frac{\tanh(\delta+\rho)}{\delta+\rho} \cdot 
  \frac{\cosh(\delta-\rho)}{\cosh(\delta-\rho) + \const}.
 \end{equation*}
 and therefore
 \begin{equation*}
 g(\delta+\rho) -  g(\delta -\rho) \le \frac{\tanh(\delta+\rho)}{\delta+\rho} 
 \left( \frac{\cosh(\delta+\rho)}{
 	\cosh(\delta+\rho) + \const} 
 - \frac{\cosh(\delta-\rho)}{
 	\cosh(\delta-\rho) + \const}\right)
 \end{equation*}
 and similarly,
  \begin{equation*}
 g(\delta+\rho) +  g(\delta -\rho) \ge 
 \frac{\tanh(\delta+\rho)}{\delta+\rho} 
 \left( \frac{\cosh(\delta+\rho)}{
 	\cosh(\delta+\rho) + \const} 
 + \frac{\cosh(\delta-\rho)}{
 	\cosh(\delta-\rho) + \const}\right).
 \end{equation*}
 As $g(x)>0$ for every $x$, we may divide these bounds and obtain via  
 elementary manipulation,
 \begin{flalign*}
\frac{ g(\delta+\rho)-g(\delta-\rho)}{ 
 	g(\delta+\rho)+g(\delta-\rho)}  & \le 
\frac{
	\frac{\cosh(\delta+\rho)}{
	 		\cosh(\delta+\rho) + \const}
 	-
 	\frac{\cosh(\delta-\rho)}{
 		\cosh(\delta-\rho) + \const}
 	}{
\frac{\cosh(\delta+\rho)}{
 	\cosh(\delta+\rho) + \const}
 +
 	\frac{\cosh(\delta-\rho)}{
 	\cosh(\delta-\rho) + \const}
}
\\ &
= \frac{\const\left[\cosh\left(\delta+\rho\right) 
	-\cosh\left(\delta-\rho\right)\right]}
{2\cosh\left(\delta+\rho\right)\cosh\left(\delta-\rho\right)
	+\const\left[\cosh\left(\delta+\rho\right)
	+\cosh\left(\delta-\rho\right)\right]}
\\ &
\le \frac{\cosh(\delta+\rho) - \cosh(\delta-\rho)}{\cosh(\delta+\rho) + 
\cosh(\delta-\rho)} = \tanh(\rho)\tanh(\delta) \le \tanh(\delta).
 \end{flalign*}
 Substituting back into~\eqref{eq:pointwise-desired-equiv} establishes the 
 desired bound. Examining the proof, we see that the bound is tight for 
 large values of $\const$ and $|\rho|$.
 \end{proof}

\renewcommand{\const}{r_0}

\subsubsection{Piecewise monotonicity of 
$\AMMW$}\label{sec:app-spec-bound-mono}

\begin{lemma}\label{lem:Ammw-mono}
	Let $\lambda\in\R^n$, $w\in\simplex[n]$ (the simplex in $\R^n$), and 
	$i,j\in[n]$ such that $\delta\defeq \half(\lambda_i - \lambda_j)>0$, and 
	set $\rho \defeq \half \log\frac{w_i}{w_j}$. When 
	$\lambda$ and $\{w_k\}_{k\ne i,j}$ are held fixed, $\AMMW_{ij}(\lambda 
	+ \log w)$ is increasing in $\rho$ for $\rho< -\delta$, and 
	 decreasing in $\rho$ for $\rho> \delta$.
\end{lemma}

\begin{proof}
	First, we write $\AMMW_{ij}(\lambda + \log w)$ explicitly as a function 
	of $\rho$, with $\lambda$ and $\{w_k\}_{k\ne i,j}$ as fixed parameters. 
	By~\eqref{eq:Aij-expression} we have
	\begin{equation*}
	\AMMW_{ij}(\lambda + \log w) = 
	\frac{\sinh( \rho + \delta )}{2(\rho+\delta)\left[\cosh(\rho + \delta) + 
	\half 
	\sum_{k\notin \{i,j\}} \frac{w_k}{\sqrt{w_i w_j}} e^{\lambda_k - 
	\frac{\lambda_i 
	+ 
	\lambda_j}{2}}\right]}.
	\end{equation*}
	Let $m = w_i+w_j = 1-\sum_{k\ne i,j}w_k$. Since $\frac{w_i}{w_j} = 
	e^{2\rho}$ 
	and $w\in\simplex[n]$, we have that $w_i = 
	\frac{m}{1+e^{-2\rho}}$ and $w_j = \frac{m}{1+e^{2\rho}}$. Therefore,
	\begin{equation*}
	\frac{1}{\sqrt{w_i w_j}} = \frac{1}{m}\sqrt{(1+e^{-2\rho})(1+e^{2\rho})}
	= \frac{2}{m}\cosh(\rho).
	\end{equation*}
	Thus, 
	\begin{equation*}
	\AMMW_{ij}(\lambda + \log w)  = \frac{\sinh( \rho + \delta 
	)}{2(\rho+\delta)\left[\cosh(\rho + \delta) + \const \cosh(\rho)\right]},
	\end{equation*}
	where $\const = \sum_{k\notin \{i,j\}} \frac{w_k}{m} e^{\lambda_k - 
	\frac{\lambda_i + \lambda_j}{2}}$ is a function of only $\lambda$ and 
	$\{w_k\}_{k\ne i,j}$, and therefore $\AMMW_{ij}(\lambda + \log w)$ can 
	be viewed as a function of $\rho$ as claimed. 
	
	Writing $x=\rho+\delta$, showing the desired monotonicity properties 
	is equivalent to showing that
	\begin{equation*}
	b(x) \defeq \frac{\sinh( x
		)}{x\left(\cosh(x) + \const \cosh(x-\delta)\right)}
	\end{equation*}
	is decreasing for $x > 2\delta$ and increasing for $x<0$. The 
	derivative of $b(x)$ is
	\begin{equation*}
	b'(x) = \frac{\cosh(x)-\frac{1}{x}\sinh(x)}{x\left(\cosh(x) + \const 
	\cosh(x-\delta)\right)}
	- \frac{\sinh(x)\left[\sinh(x) + 
	\const\sinh(x-\delta)\right]}{
	x\left[\cosh(x) + \const 
	\cosh(x-\delta)\right]^2},
	\end{equation*}
	and has, for all $x\in\R$, the same sign as
	\begin{flalign}
	\label{eq:signometer}
	s & \defeq \frac{x\left[\cosh(x) + \const 
		\cosh(x-\delta)\right] }{\sinh(x)} b'(x) 
	= \coth(x) - \frac{1}{x} -
	 \frac{\sinh(x) + \const\sinh(x-\delta)}{\cosh(x) + \const 
	\cosh(x-\delta)}.
	\end{flalign}
	For $x>2\delta$, we have by Dan's favorite inequality 
	($\frac{a_1+a_2}{b_1+b_2}\ge \min\{\frac{a_1}{b_1}, \frac{a_2}{b_2}\}$ 
	for all 
	$a_1,a_2,b_1,b_2\ge0$),
	\begin{equation*}
	\frac{\sinh(x) + \const\sinh(x-\delta)}{\cosh(x) + \const 
		\cosh(x-\delta)} \ge \min\left\{ 
	\tanh(x), \tanh(x-\delta)
	\right\}= \tanh(x-\delta) > \tanh(x/2),
	\end{equation*}
	where in the last transition we used the fact that $x>2\delta$  implies 
	$x-\delta > x/2$. Therefore, for $x>2\delta$ we have the following 
	bound for $s$, %
	\begin{equation*}
	s \le \coth(x) - \frac{1}{x} - \tanh(x/2) = 
	\frac{1}{\sinh(x)}-\frac{1}{x} < 0,
	\end{equation*}
	so we have that $b(x)$ is decreasing for $x>2\delta$ as required, since 
	$s$ has the same sign as $b'(x)$.
	
	 Similarly, for $x<0$, we have  by Dan's favorite 
	 inequality,
	 \begin{equation*}
	 \frac{-\sinh(x) - \const\sinh(x-\delta)}{\cosh(x) + \const 
	 	\cosh(x-\delta)} \ge \min\left\{ 
	 -\tanh(x), -\tanh(x-\delta)
	 \right\} = -\tanh(x).
	 \end{equation*}
	 Therefore, for $x<0$ we have
	 \begin{equation*}
	 s \ge \coth(x) - \frac{1}{x} - \tanh(x) = \frac{1}{-x} - 
	 \frac{2}{\sinh(-2x)} >
	 0,
	 \end{equation*}
	 which shows that $b(x)$ is increasing for $x<0$, concluding the proof.
\end{proof}

The following Lemma proves the intuitive fact that decreasing and 
increasing functions of the same random variable are negatively correlated.

\begin{lemma}\label{lem:mono-neg-corr}
	Let $\rho$ be a real-valued random variable, let $f,g$ be functions from 
	$\R$ to $\R$ and let $\Sset \subset \R$ be an interval. If $f(x)$ is 
	non-decreasing in $x$ for $x\in \Sset$ and $g(x)$ is 
	non-increasing in 
	$x$ for $x\in \Sset$, then
	\begin{equation*}
	\E f(\rho) g(\rho) \indic{\rho \in \Sset} \le 
	\left(\E\left[f(\rho)\mid \rho \in \Sset\right]\right) \cdot
	\left(\E g(\rho) \indic{\rho \in \Sset}\right).
	\end{equation*}
\end{lemma}

\begin{proof}
	For every $x,x'\in \Sset$ we have $(f(x)-f(x'))\cdot (g(x)-g(x')) \le 0$. 
	Hence, 
	for every $x,x'\in\R$, the bound $(f(x)-f(x'))\cdot 
	(g(x)-g(x'))\cdot\indic{x\in 
	\Sset}\indic{x'\in \Sset} \le 0$ holds as well. Let $\rho'$ be an 
	independent copy 
	of $\rho$, then
	\begin{equation*}
	\E \left[ (f(\rho)-f(\rho'))\cdot (g(\rho)-g(\rho'))\cdot \indic{\rho\in 
	\Sset}\indic{\rho'\in \Sset} \right] \le 0.
	\end{equation*}
	Rearranging and using the fact that $\rho,\rho'$ are i.i.d., we have
	\begin{equation*}
	\left(\E f(\rho) g(\rho) \indic{\rho \in \Sset}\right)  \cdot
		 \left(\E \indic{\rho' \in \Sset} \right)
	\le 
		\left(\E\left[f(\rho)\indic{\rho \in \Sset}\right]\right) \cdot
		\left(\E\left[g(\rho')\indic{\rho' \in \Sset}\right]\right).
	\end{equation*}
	Dividing by $\E \indic{\rho' \in \Sset} = \P(\rho\in \Sset)$ yields 
	the desired 
	bound.
\end{proof}

\notarxiv{
\subsection{Proof of Proposition~\ref{prop:regret-properties}, parts 
\ref{item:smoothness} and 
\ref{item:refined-smoothness}}\label{sec:app-prop11-proof}
\begin{proof}
	We begin with 
	Proposition~\ref{prop:regret-properties}.\ref{item:smoothness}: 
	\emph{for every 
	$\xM, 
	\dM\in\domainDual$, 
	$\breg{\xM}{\xM + \dM} 
	\le \frac{3}{2} \linf{\dM}^2$}.
	To show this, fix $\xM,\dM\in\domainDual$ and let 
	$p(t)\defeq\mapAv(\xM+t\dM)$. 
	The 
	Bregman divergence~\eqref{eq:bregman-def} admits the integral 
	form 
	\begin{flalign}\label{eq:bregman-hess-form}
	\breg{\xM}{\xM + \dM} 
	&= p(1)-p(0)-p'(0) 
	= \int_{0}^1 {(p'(t)-p'(0))}dt 
	= \int_{0}^1 \int_{0}^t {p''(\tau)} d\tau dt 
	\nonumber \\ &
	=
	\int_{0}^{1}\int_{0}^{t}
	{\hess \mapAv(\xM + \tau \dM)[\dM, \dM]} d\tau dt.
	\end{flalign}
	Note that since $\MapAv(\xM)\in\domain$ for every 
	$\xM\in\domainDual$, $\inner{\dM^2}{\MapAv(\xM)} \le 
	\linfs{\dM^2} 
	\lones{\MapAv(\xM)} = \linfs{\dM}^2$. Therefore, the 
	bound~\eqref{eq:main-spec-bound2}  gives
	\begin{equation*}
	\hess \mapAv(\xM + \tau \dM)[\dM, \dM] \le 
	3\linf{\dM}^2.
	\end{equation*}
	Substituting back into~\eqref{eq:bregman-hess-form} and using 
	$\int_0^1\int_0^t d\tau dt = \half$ gives 
	Proposition~\ref{prop:regret-properties}.\ref{item:smoothness}.

	Next, we show 
	Proposition~\ref{prop:regret-properties}.\ref{item:refined-smoothness}: 
	\emph{for every $\xM, 
	\dM \in\domainDual$ such that $\dM \succeq 0$ and 
	$\linf{\dM} \le \frac{1}{6}$, we have 
	$
	\breg{\xM}{\xM + \dM} \le 3\linf{\dM} 
	\inner{\dM}{\MapAv(\xM)}
	$}. When $\dM \succeq 0$, we have
	\begin{equation*}
	\inner{\dM^2}{\MapAv(\xM)} =  
	\inner{\dM}{\dM^{1/2}\MapAv(\xM)\dM^{1/2}} \le  \linfs{\dM} 
	\lones{\dM^{1/2}\MapAv(\xM)\dM^{1/2}} = \linfs{\dM} 
	\inner{\dM}{\grad \mapAv(\xM)}.
	\end{equation*}
	Plugging the bound above into the bound~\eqref{eq:main-spec-bound2} 
	and substituting back into~\eqref{eq:bregman-hess-form} gives
	\begin{equation}\label{eq:breg-pos-shift-intermediate-bound}
	\breg{\xM}{\xM + \dM} \le 3\linf{\dM} 
	\int_{0}^{1}\int_{0}^{t}
	{\inner{\dM}{\grad \mapAv(\xM + \tau \dM)}} d\tau dt.
	\end{equation}
	Moreover,
	\begin{flalign}
	\label{eq:breg-inner-integral}
	\int_{0}^{t}
	{\inner{\dM}{\grad \mapAv(\xM + \tau \dM)}} d\tau
	&=
	\int_{0}^{t}
	{p'(\tau)} d\tau
	= p(t)-p(0) 
	= \breg{\xM}{\xM + t\dM} + \inner{t\dM}{\MapAv(\xM)},
	\end{flalign}
	where the final equality uses the definition~\eqref{eq:bregman-def} of 
	the Bregman divergence. Note also that $v(t) \defeq 
	\breg{\xM}{\xM + 
		t\dM}$ is increasing for $t\ge0$ due to convexity of $\mapAv$; 
	$t v'(t) = 
	\inner{t \dM}{\grad\mapAv(\xM + t\dM) - \grad\mapAv(\xM)} \ge 0$. 
	Therefore, 
	the equality~\eqref{eq:breg-inner-integral} implies $\int_{0}^{t}
	{\inner{\dM}{\grad \mapAv(\xM + \tau \dM)}} d\tau \le 
	\breg{\xM}{\xM 
		+ 
		\dM} + t\cdot \inner{\dM}{\MapAv(\xM)}$ for every $0\le t\le 1$. 
	Substituting this back 
	into~\eqref{eq:breg-pos-shift-intermediate-bound} 
	and rearranging gives
	\begin{equation*}
	\Big(1 - 3\linf{\dM}  \Big)
	\breg{\xM}{\xM 
		+ 
		\dM}
	\le \frac{3}{2}\linf{\dM}  
	\inner{\dM}{\MapAv(\xM)}.
	\end{equation*}
	establishing part~\ref{item:refined-smoothness} of the proposition, as 
	$1 - 
	3\linf{\dM}\ge \half$ by assumption.
\end{proof}
}

\subsection{Facts about the Beta distribution}\label{sec:app-beta-facts}

Here we collect properties of Beta-distributed random variables, which we 
use in our development.

\begin{lemma}\label{lem:beta-log-expectation}
	Let $n\in\N$ and let $z\sim 
	\mathrm{Beta}(\frac{1}{2},\frac{n-1}{2})$. Then
	\begin{equation*}
	\E \log \tfrac{1}{z} = \psi\big(  \tfrac{n}{2} \big) - 
	\psi\left( \tfrac{1}{2} \right) \le 
	\log(n) + \log(2) + \gamma \le \log(4n),
	\end{equation*}
	where $\psi(x) = \frac{d}{dx}\log\Gamma(x)$ is the digamma function, 
	and $\gamma$ 
	is the Euler-Mascheroni constant.
\end{lemma}
\begin{proof}
	$\E \log \tfrac{1}{z} = \psi\big(  \tfrac{n}{2} \big) - 
	\psi\left( \tfrac{1}{2} \right)$ by the well-known formula for expectation 
	of the logarithm of a Beta random variable. We have $\psi(x) \le 
	\log(x)$~\citep{Alzer97} and $\psi(\half) = -\log(4) - \gamma$. 
	Moreover, 
	$\gamma \le 
	\log2$, giving the final bound.
\end{proof}

\begin{lemma}\label{lem:beta-logit-tail}
	Let $z\sim\mathrm{Beta}\left(\frac{1}{2},\frac{1}{2}\right)$
	and $\ell \ge 0$. Then
	\begin{equation*}
	\frac{2}{\pi}\frac{e^{-\ell/2}}{\sqrt{1+e^{-\ell}}} \le 
	\P\left(\log\frac{1-z}{z}\ge \ell\right)\le\frac{2}{\pi}e^{-\ell/2}.
	\end{equation*}
\end{lemma}
\begin{proof}
	The distribution 
	$\mathrm{Beta}\left(\frac{1}{2},\frac{1}{2}\right)$
	has density $\frac{1}{\pi}x^{-1/2}\left(1-x\right)^{-1/2}$. Therefore
	\begin{align*}
	\P\left(\log\frac{1-z}{z}\ge 
	\ell\right)=\P\left(z\le\frac{1}{1+e^{\ell}}\right)
	=\frac{1}{\pi}\int_{0}^{\left(1+e^{\ell}\right)^{-1}} 
	x^{-1/2}\left(1-x\right)^{-1/2}dx.
	\end{align*}
	To obtain a lower bound, we use $(1-x)^{-1/2} \ge 1$ for every 
	$x\in[0,1]$, and therefore,
	\begin{equation*}
	\P\left(\log\frac{1-z}{z}\ge 
	\ell\right) \ge \frac{1}{\pi}\int_{0}^{\left(1+e^{\ell}\right)^{-1}} 
	x^{-1/2}dx = \frac{2}{\pi \sqrt{1+e^\ell}} = 
	\frac{2}{\pi}\frac{e^{-\ell/2}}{\sqrt{1+e^{-\ell}}}.
	\end{equation*}
	For the upper bound, we use 
	$\left(1-x\right)^{-1/2}\le\left(1-\frac{1}{1+e^{\ell}}\right)^{-1/2}$
	for every $0\le x \le (1+e^\ell)^{-1}$, giving
	\begin{equation*}
	\P\left(\log\frac{1-z}{z}\ge 
	\ell\right) \le  
	\frac{1}{\pi}\sqrt{\frac{1+e^{\ell}}{e^{\ell}}} 
		\int_{0}^{\left(1+e^{\ell}\right)^{-1}}x^{-1/2}dx
	=
	\frac{2}{\pi}e^{-\ell/2}.
	\end{equation*}
\end{proof}

\begin{lemma}\label{lem:beta-logit-cond-exp}
	Let $z\sim\mathrm{Beta}\left(\frac{1}{2},\frac{1}{2}\right)$
	and $\ell \ge 0$. Then
	\begin{equation*}
	\expect*{\log\frac{1-z}{z} | \log\frac{1-z}{z}\ge \ell}
	\le \ell + 2\sqrt{1+e^{-\ell}}.
	\end{equation*}
\end{lemma}
\begin{proof}
	Conditional on $ \log\frac{1-z}{z}\ge \ell$, $\log\frac{1-z}{z}$ is a 
	nonnegative random variable, and we may therefore write
	\begin{flalign*}
	\expect*{\log\frac{1-z}{z} | \log\frac{1-z}{z}\ge \ell} &=
	\int_{x=0}^\infty {\P\left(\log\frac{1-z}{z} \ge x \,\middle|  \, 
	\log\frac{1-z}{z}\ge \ell \right)} dx
	\\ & = \ell + \int_{x=\ell}^\infty 
	{
	\frac{\P\left(\log\frac{1-z}{z} \ge x \right)}{
			\P\left(\log\frac{1-z}{z} \ge \ell \right)}}dx.
	\end{flalign*}
	By Lemma~\ref{lem:beta-logit-tail},
	\begin{equation*}
	\frac{\P\left(\log\frac{1-z}{z} \ge x \right)}{
		\P\left(\log\frac{1-z}{z} \ge \ell \right)}
	\le \sqrt{1+e^{-\ell}} \cdot  e^{-(x-\ell)/2}.
	\end{equation*}
	Integrating, we obtain the desired bound.
\end{proof}

\begin{lemma}\label{lem:beta-n-tail}
	Let $3\le n\in\N$ and let $z\sim 
	\mathrm{Beta}(\frac{1}{2},\frac{n-1}{2})$. For every $\delta\in(0,1)$,
	\begin{equation*}
	\P\left(z \ge \frac{\delta^2}{n}\right) > 1-\delta.
	\end{equation*}
\end{lemma}
\begin{proof}
	The random variable $z$ has density
	\begin{equation*}
	\frac{\Gamma(\frac{n}{2})}{\Gamma(\half)\Gamma(\frac{n-1}{2})} x^{-1/2}
	(1-x)^{(n-3)/2} \le \sqrt{\frac{n}{2\pi x}},
	\end{equation*}
	where we used $\Gamma(\half) = \sqrt{\pi}$ and Gautschi's inequality 
	$\Gamma(m+1)/\Gamma(m+s) \le (m+1)^{1-s}$ with $m=\frac{n}{2}-1$ 
	and $s=\half$. Integrating the upper bound on the density, we find 
	$\P(z\le \delta^2/n) \le \sqrt{\frac{2}{\pi}} \delta < \delta$.
\end{proof}

\section{Efficient computation of matrix exponential-vector  
products}
\label{sec:app-expvec}

In this section we give a more detailed discussion of matrix exponential-vector  
product approximation using the Lanczos method, and prove the results 
stated in 
Section~\ref{sec:expvec-compute}. In Section~\ref{sec:app-expvec-description} 
we formally state the Lanczos method. In Section~\ref{sec:app-expvec-known} we 
survey known approximation guarantees and derive simple corollaries. In 
Section~\ref{sec:app-expvec-mult} we show that we can apply the 
matrix exponential to a random vector with a multiplicative error guarantee, and in 
Section~\ref{sec:app-expvec-proof} we prove it implies 
Proposition~\ref{prop:expvec-main}. In Section~\ref{sec:app-expvec-improve} we 
discuss some possible improvement to our guarantees via modifications and 
alternatives to the Lanczos method. Finally, in Section~\ref{sec:app-expvec-cor} 
we prove Corollary~\ref{cor:high-probability-regret-app}.

Throughout this section we use $\matvec(\ma)$ to denote the time 
required 
to multiply the matrix $\ma$ with any vector.

\subsection{Description of the Lanczos 
method}\label{sec:app-expvec-description}

\begin{algorithm2e}
	\caption{Lanczos method for computing matrix exponential vector 
	product $\lanczos(\ma, b)$}
	\label{alg:lanczos}
	\SetAlgoLined
	\SetKwInOut{Input}{input}
	\SetKwInOut{Output}{output}
	\SetKwInOut{Return}{return}
	\SetKw{myif}{if}
	\SetKw{mythen}{then}
	\SetKw{myelse}{else}
	\SetKw{mybreak}{break}
	\Input{$\ma \in \domainDual$, number of iterations $k$, vector $b \in \R^n$}
	\BlankLine
	$q_0 \gets 0 \in \R^n$, $q_1 \gets b / \norm{b}_2$, $\beta_1 \gets 1$ 
	\BlankLine
	\For{$i = 1,\ldots, k$} 
	{

		$q_{i+1} \gets \ma q_i - \beta_i q_{i - 1}$ and $\alpha_i \gets q_{i+ 
		1}^T 
		q_i$ 
		
		\vspace{3pt}
		
		$q_{i + 1} \gets q_{i+1} - \alpha_i q_i$ and $\beta_{i + 1} = 
		\norm{q_{i + 
				1}}_2$
			
		\vspace{3pt}
		
		\myif $\beta_{i + 1} = 0$ \mythen \mybreak \myelse  $q_{i + 1} \gets 
		q_{i+ 
			1} / \beta_{i + 1}$ 
	}

	\BlankLine
	
	Let 
	\[
	\text{
	$Q = [q_1 ~ \cdots ~ q_k]$~~
	and~~
	$
		T = \left[ 
		\begin{array}{cccc}
		\alpha_1 & \beta_2 & & 0 \\
		\beta_2 & \alpha_2 & \ddots & \\
		& \ddots & \ddots & \beta_k  \\
		0 & & \beta_k & \alpha_k 
		\end{array}
		 \right]
	$}
	\] 
	
	Compute tridiagonal eigen-decomposition $T = V \Lambda V^T$ 
	 \BlankLine
	\Return{$\lanczos(\ma, b) = \norm{b}_2 \cdot QV \exp(\Lambda) 
	V^Te_1$  
	}
\end{algorithm2e}

Ignoring numerical precision issues, each iteration in the for loop requires 
$O(\matvec(\ma))$ time, and 
that for a $k$-by-$k$ tridiagonal matrix, eigen-decomposition requires 
$O(k^2)$ time 
\citep{GuEi95}, and so the total complexity is $O(\matvec(\ma)k+k^2)$. In 
practical 
settings $k\ll n \le \matvec(\ma)$ and the cost of the eigen-decomposition is 
negligible. Nevertheless, there are ways to avoid performing it, which we 
discuss 
briefly in Section~\ref{sec:app-expvec-improve}.

\subsection{Known approximation results, and 
	some corollaries}\label{sec:app-expvec-known}

We begin with a result on uniform polynomial approximation of the exponential 
due to \cite{SachdevaVi14}.

\begin{theorem}
	[\cite{SachdevaVi14}, Theorem 4.1 Restated]
	\label{thm:uniform_approx}
	For every $b > 0$ and every $\epsilon \in (0, 1]$ there exists polynomial 
	$p : 
	\R \rightarrow \R$ of  degree 
	$O(\sqrt{\max\{b, \log(1/\epsilon)\} \log(1/\epsilon)})$ such that
	\[
	\sup_{x \in [0, b]} |\exp(-x) - p(x) | \leq \epsilon \, .
	\]
\end{theorem}

As an immediate corollary of this we obtain the following bounds for  
approximating 
$\exp(x)$ over arbitrary values

\begin{corollary}
	\label{cor:approx_exp}
	For every $a < b \in \R$ and every $\epsilon \in (0, 1]$ there exists 
	polynomial $p : \R \rightarrow \R$ of  
	degree $O(\sqrt{\max\{b - a, \log({1}/{\epsilon})\} 
	\log({1}/{\epsilon})})$ 
	polynomial such that
	\[
	\sup_{x \in [a, b]} |\exp(x) - p(x) | \leq \epsilon \exp(b) \,.
	\]
\end{corollary}

\begin{proof}
	For all  $x \in [a, b]$ we have $b - x \in [0, b - a]$ and therefore by 
	Theorem~\ref{thm:uniform_approx} there is a degree $O(\sqrt{\max\{b - 
		a, \log(1/\epsilon)\} \log(1/\epsilon)})$ polynomial $q : \R \rightarrow 
	\R$ such that 
	\[
	\sup_{x \in [a, b]} |\exp(-(b - x)) - q(b - x) | \leq \epsilon \,.
	\]
	Since $\exp(-(b - x)) = \exp(-b) \exp(x)$, the polynomial $p(x) = \exp(b) q (b 
	-x)$ is as desired.
\end{proof}

The classical theory on the Lanczos method tells us that its error is bounded by 
twice that of any uniform polynomial approximation. However, this theory does not 
account for finite precision. A recent result~\citep{MuscoMuSi18} ties 
polynomial 
approximation to 
the error of the Lanczos method using finite bitwidth floating point operations.

\begin{theorem}[\cite{MuscoMuSi18}, Theorem 1]
	\label{thm:lanczos}
	Let $\ma  \in \domainDual$, $u \in \R^n$, 
	and $f : \R \rightarrow \R$. Suppose $k \in \N$, $\eta \in (0, 
	\opnorm{\ma}]$ and a polynomial $p$ for degree $<k$ satisfy,
	\[
	\sup_{x \in [\lambdamin(\ma) - \eta , \lambdamax(\ma)+ \eta]} |f(x) - 
	p(x)| \leq \epsilon_k
	~~\mbox{and}~~
	\sup_{x \in [\lambdamin(\ma) - \eta , \lambdamax(\ma)+ \eta]} |f(x)| 
	\le C.
	\]
	For any $\mu\in(0,1)$, let $y_{k,\mu}$ be the output of $k$ iterations of 
	the Lanczos method for approximating  
	$f(\ma)v$, using floating point operations with $B \ge c  \log 
	(\frac{nk 
		\opnorm{\ma}}{\mu \eta})$ bits precision (for numerical constant 
	$c<\infty$). 
	Then $y_{k,\mu}$  satisfies 
	\[
	\ltwo{f(\ma) u - y_{k,\mu}}
	\leq (7k \cdot \epsilon_k + \mu \cdot  C) \ltwo{u}.
	\]
	If arithmetic operations with $B$ bits of precision can be performed in  
	$O(1)$ time then the method can be implemented in time 
	$O(\matvec(A)k + k B \max\{k, B\})$.
\end{theorem}

Specializing to the matrix exponential and using the uniform approximation 
guarantee of Corollary~\ref{cor:approx_exp}, we immediately obtain the following.

\begin{corollary}
	\label{cor:lanczos-exponential}
	Let $\ma \in \domainDual$, $u \in \R^n$, and 
	$\epsilon > 0$, and set $M 
	= \max\{\opnorm{\ma} , \log(1/\epsilon), 1\}$. There exists numerical 
	constants $c,c'<\infty$ such that, for $k\ge c\sqrt{M \log(M/\epsilon)}$ 
	and 
	$B\ge c' \log(\frac{nM}{\epsilon})$, computing $y=\lanczos(A,u)$ with 
	$B$ bits of floating point precision guarantees
	\[
	\norm{\exp(\ma) u - y}_2 \leq \epsilon \exp(\lambdamax(\ma)) 
	\ltwo{u}.
	\]
	The computation takes time
	\[
	O\left(
	\matvec(\ma)\sqrt{M \log (M/\epsilon)} + M \log^2 (nM/\epsilon)
	\right)
	\]
	provided $\Theta(\log(\frac{nM}{\epsilon}))$ bit arithmetic operations 
	can be performed in time $O(1)$. 
\end{corollary}

\begin{proof}
	Let $\eta = 1$. Using 
	$\lambdamax(\ma) - \lambdamin(\ma) \leq 2 \norm{\ma}$,  
	Corollary~\ref{cor:approx_exp} yields that for all $\alpha \in (0, 1]$ there 
	exists a degree $O\left(\sqrt{\max \{1+\opnorm{\ma}, 
	\log(\frac{1}{\alpha})\} 
		\log(\frac{1}{\alpha})}\right)$ polynomial $p : \R \rightarrow \R$ such 
		that
	\[
	\sup_{x \in [\lambdamin(\ma) - \eta, \lambdamax(\ma) + \eta]} |\exp(x) - 
	p(x) | \leq \alpha \exp(\eta) \exp(\lambdamax(\ma)) ~.
	\]
	Further, since $|\exp(x)| \leq \exp(\eta) \exp(\lambdamax(\ma))$ for all 
	$x \in [\lambdamin(\ma) - \eta, \lambdamax(\ma) + \eta]$, 
	Theorem~\ref{thm:lanczos} with $f(x)=e^x$ and $\eta=1$ implies that 
	for 
	all $\mu \in (0, 1)$, after applying Lanczos for $k = 
	O(\sqrt{\max\{\opnorm{\ma}, \log(1/\alpha)\} \log(1/\alpha)})$ 
	iterations  on a floating point machine with $\Theta(B)$ bits 
	of precision for $B = \log(\frac{n k \norm{\ma}}{\mu})$ returns $y$ with
	\[
	\norm{f(\ma) u - y}_2 \leq \left(\mu  + \alpha \cdot 
	O(\sqrt{
		\max\{\opnorm{\ma}, \log(1/\alpha)\} \log(1/\alpha)})  \right) 
	\exp(\eta) 
	\exp(\lambdamax(\ma)) )
	\]
	in time $O((\matvec(\ma) +n)k + k B \max\{k, B\})$. Choosing, $\alpha 
	= 
	O(\epsilon / (M \log(M/\epsilon)))$ and $\mu = O(\epsilon)$ yields 
	the result.
\end{proof}

\subsection{Multiplicative approximation for random 
vectors}\label{sec:app-expvec-mult}

We now combine the known results cited in the previous section with the 
randomness of the vector fed to the matrix exponential, to obtain a multiplicative 
guarantee that holds with high-probability over the choice of 
$u$, but not for all $u\in\sphere^{n-1}$.

\begin{proposition}
	\label{prop:random-exp-vec}
	Let $\epsilon \in (0,1)$, $\delta\in(0,1)$, and $\ma \in \domainDual$. If 
	$u$ 
	is sampled uniformly at random from the unit sphere and for
	$k = \Omega(\sqrt{M \log(n M/(\epsilon \delta)}) \in \N$ for $M = 
	\max\{\opnorm{\ma} , \log(n/(\epsilon\delta)), 1\}$ we 
	let $y = \lanczos(\ma, u)$ (See Algorithm~\ref{alg:lanczos}) then 
	\[
	\norm{\exp(\ma)u - y}_2 \leq \epsilon \norm{\exp(\ma) u}_2
	~\mbox{with probability}~
	\ge 1-\delta.
	\]
	This can be implemented in time 
	$
	O\left(
	\matvec(\ma)  \sqrt{M \log (nM/(\epsilon \delta)} + M \log^2 
	(nM/(\epsilon \delta))
	\right) 
	$
	on a floating point machine with $O(\log(n M / (\epsilon \delta)))$ bits of 
	precision where arithmetic operations take $O(1)$ time.
\end{proposition}

\begin{proof} %
	Consider an application of Corollary~\ref{cor:lanczos-exponential} to 
	compute $y$ such that 
	\[
	\norm{\exp(\ma) u - y}_2 \leq \epsilon' \exp(\lambdamax(\ma)) 
	\norm{u}_2 ~.
	\]
	Now let $v$ be a unit eigenvector of $\ma$ with eigenvalue 
	$\lambdamax(\ma)$. Since $v$ is an eigenvector or the PSD matrix 
	$\exp(\ma)$ with eigenvalue $\exp(\lambdamax(\ma))$ we have that 
	$\norm{\exp(\ma) u} \geq \exp(\lambdamax) \left|v^T u\right|$. 
	However, 
	since $u$ is a random unit vector we have that $|v^T u|^2 / \ltwo{u}^2 
	\sim \mathrm{Beta}(\half,\frac{n-1}{2})$.  
	Lemma~\ref{lem:beta-n-tail} 
	therefore gives that  $|v^T u|^2/\ltwo{u}^2 \ge \frac{\delta^2}{n}$ with 
	probability 
	at 
	least $1-\delta$. Consequently, 
	$\exp(\lambdamax(\ma)) \norm{u}_2 \leq \frac{\sqrt{n}}{\delta}
	\norm{\exp(\ma)u}_2$ with the same probability. 
	Choosing $\epsilon' = \epsilon \delta/ \sqrt{n}$ and invoking 
	Corollary~\ref{cor:lanczos-exponential}  yields the result.
\end{proof}

\subsection{Proof of 
Proposition~\ref{prop:expvec-main}}\label{sec:app-expvec-proof}

The following lemma relates the multiplicative approximation error for  
matrix 
exponential vector products with the additive approximation error for 
$\Map(\xM)$ under trace norm. Combining it with 
Proposition~\ref{prop:random-exp-vec} immediately yields 
Proposition~\ref{prop:expvec-main}.

\begin{lemma}
	Let $\xM \in \domainDual$, $u, y \in \R^n$ and $\epsilon \in [0, 1)$.  
	If $y \in \R^n$ satisfies 
	\[
	\norm{\exp(\xM/2) u - y}_2 \leq \frac{\epsilon}{\sqrt{8}} 
	\norm{\exp(\xM/2) u }_2 
	\]
	then
	\[
	\lone{\Map(\xM) - \frac{y y^T}{\ltwo{y}^2}} \le \epsilon ~.
	\]
\end{lemma}

\begin{proof}
	Let $z \defeq \exp(\xM / 2) u$ so that by assumption $\norm{z - y}_2 
	\leq \epsilon \norm{z}_2$. Further, let $\bar{z} \defeq z / \ltwo{z}$ and 
	$\bar{y} \defeq y / \ltwo{y}$. Direct calculation (see e.g. Lemma 27 of \cite{CohenLeMiPaSi16}) yields that the eigenvalues of $\bar{z}\bar{z}^T 
	- \bar{y}\bar{y}^T$ are $\pm  \sqrt{1 - (\bar{z}^T \bar{y})^2}=\pm \half  
	\ltwo{\bar{z}+\bar{y}} \ltwo{\bar{z}-\bar{y}}$ and 
	therefore the definition of $\Map(\xM)$ yields 
	\begin{equation}
	\label{eq:pu_exp_1}
	\lone{\Map(\xM) - \frac{y y^T}{\norm{y}_2^2}}
	= \lone{\bar{z} \bar{z}^T - \bar{y} \bar{y} ^T}
	= \ltwo{\bar{z}+\bar{y}}\cdot \ltwo{\bar{z}-\bar{y}}
	\leq \sqrt{2} \norm{\bar{z} - \bar{y}}_2,
	\end{equation}
	where in the last inequality we used that $\bar{z}$ and $\bar{y}$ are unit 
	vectors. 
	Further, by the triangle inequality and the definitions of $\bar{y}$ and 
	$\bar{z}$ we have
	\begin{align}
	\norm{\bar{z} - \bar{y}}_2 
	& \leq
	\norm{\frac{z}{\norm{z}_2} - \frac{y}{\norm{z}_2}}_2
	+ 
	\norm{\frac{y}{\norm{z}_2} - \frac{y}{\norm{y}_2}}_2
	\nonumber \\ & = 
	\frac{\norm{z - y}_2}{\norm{z}_2} 
	+ \frac{|\norm{y}_2 - \norm{z}_2|}{\norm{z}_2}
	\le 2 \frac{\norm{z - y}_2}{\norm{z}_2} 
	\label{eq:pu_exp_2}
	\end{align}
	Combining \eqref{eq:pu_exp_1} and \eqref{eq:pu_exp_2} with the fact 
	that $\norm{z - y}_2 \leq (\epsilon/\sqrt{8}) \norm{z}_2$ then yields
	\[
	\lone{\Map(\xM) - \frac{y y^T}{\norm{y}_2^2}}
	\leq 
	\sqrt{2}\cdot 2\cdot (\epsilon/\sqrt{8}) 
	=\epsilon.
	\]
\end{proof}
\noindent
Therefore, Proposition~\ref{prop:expvec-main} follows immediately by 
invoking~\ref{prop:random-exp-vec} with slightly smaller $\epsilon$.

\subsection{Improvements to the Lanczos 
method}\label{sec:app-expvec-improve}

In this paper we focused on the Lanczos method for approximating matrix 
exponential vector products because of its excellent practicality and clean analysis. 
However, there are several modifications to the method with appealing 
features, which we now describe briefly. 
A common theme among these modifications is the use of rational 
approximations to 
the exponential, which converge far faster than polynomial approximations 
\citep{OrecchiaSaVi12,SachdevaVi14}. 
Consequently, it suffices to perform $\Otil{1}$ Lanczos iterations on a 
carefully shifted and inverted version of the matrix. Each of these iterations 
then 
involves solving a linear system, and the efficacy of the shift-invert scheme will 
depend on how quickly they are solved.

One basic approach to solving these systems is via standard iterative methods, 
e.g.\ conjugate gradient. We expect such approach to offer little to no 
advantage 
over applying the Lanczos approximation directly, as both methods produce 
vectors in the same Krylov subspace. However, the approach renders the 
number of 
Lanczos iterations $k$ logarithmic in $\opnorm{\ma}$, and 
therefore 
the cost $k^2$ will never dominate the cost of the matrix-vector 
products~\citep[Corollary 17]{OrecchiaSaVi12,MuscoMuSi18}. 

There is, however,  a simpler way of avoiding the 
eigen-decomposition---simply 
use the rational approximation on the tridiagonal matrix formed by running 
the ordinary Lanczos method, as \cite{Saad92} proposes. With an 
appropriate rational function, computing a highly accurate approximation  
to $\exp(T) e_1$ requires $\Otil{1}$ 
tridiagonal system solves, each costing $O(k)$ time. We leave the 
derivation 
of explicit error bounds for this technique (similar to  
Corollary~\ref{cor:approx_exp}) to future work. In practice, the cost 
$O(k^2)$ of 
tridiagonal eigen-decomposition will often be very small compared to the 
cost 
$O(\matvec(\ma)k)$ of the matrix-vector products.

More significant improvements are possible if the linear system solving 
routine is 
able to exploit information beyond matrix-vector products. For example, 
consider 
the case where the matrix to be exponentiated is a sum of very sparse 
matrices---this will happen for our sketch whenever the $\g$ matrices 
are 
much 
sparser than 
their cumulative sum. Then, it is possible to use stochastic variance 
reduced 
optimization methods to solve the linear system, as~\cite{AllenLi17} 
describe. 
Another scenario of interest is when the input matrix has a Laplacian/SDD 
structure and in this case the performance of specialized linear system 
solvers 
implies approximation guarantees where the 
polynomial dependence on $\opnorm{\ma}$ is removed altogether 
\citep{OrecchiaSaVi12}.
A final useful structure is a chordal sparsity pattern~\citep{VandenbergheAn15}, 
which enables efficient 
linear system solving through fast Cholesky decomposition.

\subsection{Proof of 
Corollary~\ref{cor:high-probability-regret-app}}\label{sec:app-expvec-cor}

\arxiv{
\restateHighProbApp*
}%
\notarxiv{
\begin{restatable}{corollary}{restateHighProbApp}
	\label{cor:high-probability-regret-app}
	Let $\g[1],\ldots,\g[T]$ be symmetric gain matrices satisfying 
	$\linf{\g}\le 1$  
	for every $t$. There exists a
	numerical constant $k_0<\infty$, such that for  every $T\in\N$ and 
	$\delta \in (0,1)$,  
	$\xApp[t;k_t]$ defined in~\eqref{eq:lanczos-proj-def} with 
	$k_t=\ceil{k_0 (\sqrt{1+\eta t}) \log(\tfrac{nT}{\delta})}$, and 
	$\x$ 
	defined in~\eqref{eq:rand-da} satisfy
	\begin{equation}\label{eq:regret-app-bound}
	\sum_{t=1}^T 
	\inner{\g}{\xApp[t;k_t]}  \ge -1 + \sum_{t=1}^T 
	\inner{\g}{\x} 
	~~\mbox{w.p.}\ge 1-\delta/2.
	\end{equation}
	Let $\eps\in(0,1]$,  
	$T = \frac{16\log(4en/\delta)}{\eps^2}$ and 
	$\eta = \sqrt{\frac{2\log(4en)}{3T}}$. If 
	Assumption~\ref{ass:bandit-adversary} holds with respect to the 
	actions 
	$\xApp[t;k_t]$, then with probability at least 
	$1-\delta$, 
	$\frac{1}{T}\lambdamax\left( \sum_{i=1}^T \g \right) - 
	\frac{1}{T}\sum_{t=1}^T \inner{\g}{\xApp[t;k_t]} \le \eps$. Computing 
	the actions 
	$\xApp[1;k_1], \ldots, 
	\xApp[T;k_T]$ requires 
	$O(\eps^{-2.5}\log^{2.5}(\frac{n}{\epsilon\delta}))$ matrix-vector 
	products.
\end{restatable}
}

\begin{proof}
	To obtain the bound~\eqref{eq:regret-app-bound} we use 
	Proposition~\ref{prop:expvec-main} with $\epsilon\gets\frac{1}{T}$ and 
	$\delta\gets\delta/(2T)$ (since we will use a union bound). At iteration $t$, 
	$\linfs{\g[i]}\le1$ for all $i<t$, the 
	quantity $M$ appearing in Proposition~\ref{prop:expvec-main} can be 
	bounded as 
	\begin{equation*}
	M \le \left(1+\linf{\frac{\eta}{2} \sum_{i=1}^{t-1}\g[i]}\right)
	\log{\frac{nT^2}{\delta}}
	\le O(1) (1+\eta t)\log{\frac{nT}{\delta}}.
	\end{equation*}
	Therefore, our choice of $k_t $ suffices to guarantee, for $\xM_t = \eta 
	\sum_{i=1}^{t-1}\g[i]$, 
	\begin{equation*}
	\lones{\Map[u_t](\xM_t)-\MapApp[u_t;k_t](\xM_t)} \le \frac{1}{T}
	~\mbox{with probability}~
	\ge 1-\frac{\delta}{2T},
	\end{equation*}
	and so by the union bound the inequality above holds for all $t=1,\ldots,T$ 
	with probability at least $1-(\delta/2)$. Note that when using 
	Proposition~\ref{prop:expvec-main} we use the fact that $u_t$ is 
	independent of $\xM_t$. Thus, we have
	\begin{equation*}
	\sum_{t=1}^T \inner{\g}{\x - \xApp[t;k_t]} \le 
	\sum_{t=1}^T \linf{\g}\lones{\x - \xApp[t;k_t]}
	\le \sum_{t=1}^T \lones{\Map[u_t](\xM_t)-\MapApp[u_t;k_t](\xM_t)} =1,
	\end{equation*}
	giving~\eqref{eq:regret-app-bound}, where we have used $\linfs{\g}\le 1$ for 
	every $t$.
	
	Note that if Assumption~\ref{ass:bandit-adversary} holds with respect to 
	the actions $\xApp[t;k_t]$ then we have $\g \perp u_t \mid \Filt$ and 
	therefore $\expect*{\inner{\g}{\x}|\Filt} = \inner{\g}{\xAv}$ so that 
	Corollary~\ref{cor:high-probability-regret} holds. Thus, to obtain the 
	second part of the corollary, we use the 
	bound~\eqref{eq:regret-bound-basic-high-prob} with $\delta\gets\delta/2$ 
	and $\eta$ and $T$ as specified; using a union bound again we have 
	that~\eqref{eq:regret-app-bound} 
	and~\eqref{eq:regret-bound-basic-high-prob} hold 
	together with probability at least $1-\delta$. Note that $\eta \le \eps \le 1$ 
	and therefore $1/T \le 1/(\eta T)$. This gives,
	\begin{flalign*}
	&\frac{1}{T}\lambdamax\left( \sum_{i=1}^T \g \right) - 
	\frac{1}{T}\sum_{t=1}^T \inner{\g}{\xApp[t;k_t]} \le 
	\frac{1}{T} + \frac{3 \eta}{2} + \frac{\log(4n)}{\eta T} 
	+ \sqrt{\frac{2\log{\tfrac{2}{\delta}}}{T}}
	\\ & \hspace{28pt}
	\le \frac{3\eta}{2} + \frac{\log(4en)}{\eta T} 
	+ \sqrt{\frac{2\log{\tfrac{2}{\delta}}}{T}}
	= 
	\sqrt{\frac{6\log(4en)}{T}} 
	+ \sqrt{\frac{2\log{\tfrac{2}{\delta}}}{T}} \le \eps,
	\end{flalign*}
	as required. Finally note that $1+\eta T = 
	O(\epsilon^{-1}\log(\frac{n}{\delta}))$ and consequently 
	\begin{equation*}
	k_T = O\left(\epsilon^{-1/2}\log^{1/2}( 
	\tfrac{n}{\delta})\log^{1/2}(\tfrac{nT}{\delta})\right)
	= O\left(\epsilon^{-1/2}\log^{1.5}(\tfrac{n}{\epsilon\delta})\right).
	\end{equation*}
	Since $k_1 \le k_2 \le \cdots k_T$, the total number of matrix-vector products 
	is bounded by $T \cdot 
	k_T=O(\eps^{-2.5}\log^{2.5}(\frac{n}{\epsilon\delta}))$, which concludes 
	the proof.
\end{proof}

\end{document}